\documentclass{article}

\usepackage{microtype}
\usepackage{graphicx}
\usepackage{subfigure}
\usepackage{booktabs} % for professional tables

\usepackage{hyperref}

\newcommand{\kmeans}{$k$-means}
\newcommand{\cut}[1]{}

\usepackage[accepted]{icml2020}

\usepackage{array}
\usepackage{multirow}
\usepackage{booktabs}
\usepackage{epstopdf}
\usepackage{mdwmath}
\usepackage{mdwtab}
\usepackage{amsthm}
\usepackage{amssymb}
\usepackage{amsmath}

\newtheorem{theorem}{Theorem}

\icmltitlerunning{SimpleMKKM: Simple Multiple Kernel K-means}

\begin{document}

\twocolumn[
\icmltitle{SimpleMKKM: Simple Multiple Kernel K-means}



\icmlsetsymbol{equal}{*}

\begin{icmlauthorlist}
\icmlauthor{Xinwang Liu}{NUDT}
\icmlauthor{En Zhu}{NUDT}
\icmlauthor{Jiyuan Liu}{NUDT}
\icmlauthor{Timothy Hospedales}{EU}
\icmlauthor{Yang Wang}{HUT}
\icmlauthor{Meng Wang}{HUT}
\end{icmlauthorlist}

\icmlaffiliation{NUDT}{School of Computer, National University of Defense Technology}
\icmlaffiliation{EU}{Edinburgh University}
\icmlaffiliation{HUT}{Hefei University of Technology}

\icmlcorrespondingauthor{Xinwang Liu}{xinwangliu@nudt.edu.cn}

\icmlkeywords{Machine Learning, ICML}

\vskip 0.3in
]



\printAffiliationsAndNotice{} % otherwise use the standard text.

\begin{abstract}

We propose a simple yet effective multiple kernel clustering algorithm, termed simple multiple kernel \kmeans{} (SimpleMKKM). It extends the widely used supervised kernel alignment criterion to multi-kernel clustering. Our criterion is given by an intractable minimization-maximization problem in the kernel coefficient and clustering partition matrix. To optimize it, we re-formulate the problem as a smooth minimization one, which can be solved efficiently using a reduced gradient descent algorithm. We theoretically analyze the performance of SimpleMKKM in terms of its clustering generalization error.  Comprehensive experiments on 11 benchmark datasets demonstrate that SimpleMKKM outperforms  state of the art multi-kernel clustering alternatives.

% We propose a simple while effective multiple kernel clustering algorithm, simple multiple kernel \kmeans{} (SimpleMKKM).
% Theoretical analysis and experimental study demonstrate its effectiveness.

\end{abstract}

\section{Introduction}\label{Intro}

In multiple kernel clustering (MKC) \cite{ZhaoKZ09}, we aim to combine a set of pre-specified kernel matrices to improve clustering performance. These kernel matrices could encode heterogeneous sources or views of the data \citep{YuTLGSMM12}. One popular method, multiple kernel \kmeans (MKKM) \cite{HuangCC12}, has been  studied intensively and used in various applications \cite{YuTLGSMM12,GonenM14,LiuZLWZLKSYG2019,HLZZ19,KumarD11,KumarRD11}. The approach is attractive also from a theoretical perspective, as it unifies the search of the optimal base kernel coefficient and the clustering partition matrix into a single objective function, which is usually solved by using two-step alternating optimization on the coefficients and clustering partition matrix.

Several variants of MKKM have been developed to further improve the clustering performance \cite{YuTLGSMM12,GonenM14,LiL0DYZ16,LiuDY0Z16,LiuLWDYZ17,LiuZLWZLKSYG2019}.
%For example, an optimized kernel k-means algorithm (OKKC) is proposed  in  to combine multiple data sources, where an alternating minimization framework is used to optimize the cluster membership and kernel coefficients.
Notably, \citet{GonenM14} substantially increase the expressiveness of MKKM by allowing for a locally adaptive kernel mixtures, which can better capture sample-specific characteristics of the data. \citet{LiL0DYZ16} propose an extension that optimizes a localized kernel alignment criterion. It aligns the local density of the samples given by the $k$-nearest neighbours with an ideal similarity matrix.
This alignment helps to keep neighbouring sample pairs together, which avoids unreliable similarity evaluation.
\cut{Such an alignment helps the clustering algorithm to focus on neighboring sample pairs, in that they shall stay together. This avoids unreliable similarity evaluation for farther sample pairs.} %Too verbose for intro.
Observing that existing MKKM algorithms do not sufficiently consider the correlation among these kernels,  \citet{LiuDY0Z16} employ matrix regularization to reduce the redundancy and enhance the diversity of the selected kernels. Most of existing MKKM algorithms assume that the optimal kernel is a linear combination of a group of base kernels. This assumption is challenged in \citet{LiuLWDYZ17}, who propose an optimal neighborhood kernel clustering (ONKC) algorithm to enhance the representability of the optimal kernel and strengthen the negotiation between kernel learning and clustering. More recently, MKKM algorithms have been extended to handle missing views \cite{LiuZLWZLKSYG2019}. By assuming the optimal kernel is a linear combination of the base kernel matrices, \citet{Bang2018robust} develop a minimization-maximization framework that aims to be robust to adversarial perturbation. All these variants potentially improve standard MKKM and achieve promising clustering performance in various applications.

The objective functions of the mentioned methods differ, but they all share one commonality: they learn the kernel coefficient and the clustering partition matrix \emph{jointly}. By this way, the leaned kernel coefficient can best serve the clustering, leading to superior clustering performance. However, simultaneously solving for the kernel coefficients \emph{and} the clustering partition is intractable. One commonly adopted remedy is to decouple the optimization of the kernel coefficients and the clustering partition through a block coordinate descent algorithm, which optimizes the two alternately.  This means, one block of variables is minimized while the other is kept fixed. %This alternate optimization framework, on one hand, provides an efficient approach to learn a local optimal solution. On the other hand,
However, such alternate optimization algorithms can get trapped into a local optima of the objective function. As a remedy, \citet{LiuDY0Z16,LiL0DYZ16} propose regularization strategies to avoid getting trapped into local minimum. The incorporation of these regularization terms comes at a price: the approach has additional hyper-parameters, which are difficult to select, given the unsupervised nature of clustering tasks.

In this paper, we propose Simple MKKM (SimpleMKKM)---a novel formulation for multiple kernel clustering that addresses the aforementioned shortcomings.  Unlike previous approaches, % minimizing (maximizing) the kernel coefficient and clustering matrix simultaneously,
SimpleMKKM optimizes the unsupervised kernel alignment criterion directly. Specifically, it minimizes kernel alignment with respect to the kernel coefficient and maximizes it with respect to the clustering matrix. This minimization-maximization optimization problem cannot readily be solved using existing alternate optimization frameworks. However, we show that this min-max problem actually leads to a more efficient and effective optimization algorithm. Specifically, we reformulate the min-max problem as a minimization problem, whose objective relies on the known optimal solution to kernel \kmeans. We then prove the differentiability of the optimal value function and calculate its reduced gradient. This leads to a solution using a reduced gradient descent algorithm, without alternating optimization. We show a generalization error bound for our approach, thus theoretically guaranteeing its clustering performance. We conduct comprehensive experiments on eleven benchmark datasets, where we compare SimpleMKKM to eight baseline methods in terms of three common evaluation criteria. We observe that SimpleMKKM consistently outperforms its competitors.
%We summarize the main contributions of this paper as follows:
%\begin{itemize}
%\item
%
%\item
%
%\item
%\end{itemize}

\section{Related Work}

In this section, we briefly review the most related, including multiple kernel \kmeans{} (MKKM) and robust MKKM clustering using min-max optimization \cite{Bang2018robust}.

\subsection{MKKM}

Given a group of pre-calculated kernel matrices $\{\mathbf{K}_{p}\}_{p=1}^{m}$, MKKM assumes that the optimal kernel matrix $\mathbf{K}_{\boldsymbol{\gamma}}$ can be parameterized as $\mathbf{K}_{\boldsymbol{\gamma}}=\sum_{p=1}^{m}\gamma_{p}^2\mathbf{K}_{p}$, where $\boldsymbol{\gamma}\in\Delta=\{\boldsymbol{\gamma}\in\mathbb{R}^{m}|\sum_{p=1}^{m}\gamma_{p}=1,\,\gamma_{p}\geq 0,\,\forall p\}$ represents the kernel weights of these base kernel matrices. It jointly learns the kernel weights $\boldsymbol{\gamma}$ and the clustering partition matrix $\mathbf{H}$ by optimizing Eq. (\ref{eq:MKKM_MinMin}).
\begin{equation}\label{eq:MKKM_MinMin}
\small{
\begin{split}
\min\nolimits_{\boldsymbol{\gamma}\in\Delta}\;\min\nolimits_{\mathbf{H}}\;&\;\mathrm{Tr}\left(\mathbf{K}_{\boldsymbol{\gamma}}(\mathbf{I}-\mathbf{H}\mathbf{H}^{\top})
\right)\\
s.t.\;&\;\mathbf{H}\in\mathbb{R}^{n\times k},\,\mathbf{H}^{\top}\mathbf{H}=\mathbf{I}_{k}.
\end{split}
}
\end{equation}

In literature, the optimization problem in Eq. (\ref{eq:MKKM_MinMin})  is usually be solved by alternatively updating $\mathbf{H}$ and $\boldsymbol{\gamma}$: (i) \textbf{Optimizing $\mathbf{H}$ given $\boldsymbol{\gamma}$}. With the kernel coefficients $\boldsymbol{\gamma}$ fixed, $\mathbf{H}$ can be obtained by solving a kernel $k$-means clustering optimization problem; (ii) \textbf{Optimizing $\boldsymbol{\gamma}$ given $\mathbf{H}$}. With $\mathbf{H}$ fixed, $\boldsymbol{\gamma}$ can be optimized via solving the following quadratic programming with linear constraints,
\begin{equation}\label{eq:MKKM:kw}
\small{\begin{split}
\min\nolimits_{\boldsymbol{\gamma}\in\Delta}\;&\;\sum\nolimits_{p=1}^{m}\gamma_{p}^{2}\mathrm{Tr}\left(\mathbf{K}_{p}
(\mathbf{I}_{n}-\mathbf{H}\mathbf{H}^{\top})\right),
\end{split}
}
\end{equation}
which has a closed-form solution.

As noted in \citet{YuTLGSMM12,GonenM14}, using a convex combination of kernels $\sum_{p=1}^{m}\gamma_{p}\mathbf{K}_{p}$ to replace $\sum_{p=1}^{m}\gamma_{p}^2\mathbf{K}_{p}$ is not a viable option, because this could make only one single kernel activate and all the others assigned with zero weight, as seen from Eq. (\ref{eq:MKKM:kw}). Other recent work using $\ell_2$-norm combinations can be found in \citet{KloftBSZ11,CortesMR09,LiuZLWZLKSYG2019}.

\subsection{Robust MKKM Using Min-Max Optimization}

Recently, \citet{Bang2018robust} proposed a MKKM clustering method with the aim to be robust against adversarial perturbation. To achieve this goal, the authors use a $\min_{\mathbf{H}}$-$\max_{\boldsymbol{\gamma}}$ formulation that combines views so as to achieve high within-cluster variance in the combined space $\mathbf{W}_{\boldsymbol{\gamma}}$ and then updates clusters by minimizing such variance. Its optimization problem is,
\begin{equation}\label{eq:MKKM_MM}
\small{
\begin{split}
\min\nolimits_{\mathbf{H}}\;\max\nolimits_{\boldsymbol{\gamma}\in\Theta}\;&\;\mathrm{Tr}\left(\mathbf{W}_{\boldsymbol{\gamma}}(\mathbf{I}-\mathbf{H}\mathbf{H}^{\top})
\right)\\
s.t.\;&\;\mathbf{H}\in\mathbb{R}^{n\times k},\,\mathbf{H}^{\top}\mathbf{H}=\mathbf{I}_{k},
\end{split}
}
\end{equation}
where $\Theta=\{\boldsymbol{\gamma}\in\mathbb{R}^{m}|\sum_{p=1}^{m}\gamma_{p}^2\leq1,\,\gamma_{p}\geq 0,\,\forall p\}$ and $\mathbf{W}_{\boldsymbol{\gamma}}=\sum_{p=1}^{m}\gamma_{p}\mathbf{K}_{p}$.

Note that in contrast to Eq. (\ref{eq:MKKM_MinMin}), the above approach adopts an $\ell_{2}$-norm constraint on the kernel weights to avoid sparse solutions. It is observed that using an $\ell_2$-norm constraint can obtain non-sparse kernel coefficients, which is helpful to better utilize the complementary information in the data. Similar to MKKM, the problem in Eq. (\ref{eq:MKKM_MM}) can be solved by following the same alternate optimization framework.

Although the objective functions of MKKM and its variants may vary, they share a common alternate optimization routine. The aforementioned alternate framework could cause the optimization w.r.t $\boldsymbol{\gamma}$ to produce high redundant or overly sparse solutions \cite{LiuDY0Z16}. This in turn would make the multiple kernel matrices less utilized, and adversely affects the clustering performance. A direct remedy is to incorporate some regularization on $\boldsymbol{\gamma}$ to help its optimization \cite{LiuDY0Z16,LiL0DYZ16}. However, the incorporation of regularization may introduce extra hyper-parameters. How to determine those in unsupervised learning tasks such as clustering is difficult. In the following, we introduce our simple MKKM objective, and design a novel optimization procedure for it that avoids these issues.

\section{SimpleMKKM: Simple MKKM}\label{submission}

In this section, we first give the proposed SimpleMKKM kernel alignment-based objective. We then reformulate it as the minimization of an optimal value function, and prove its differentiability. After that, we develop a reduced gradient descent algorithm to solve it efficiently and effectively.

\subsection{SimpleMKKM Formulation}\label{sec:formulation}

Kernel alignment criterion has been widely used for kernel tuning in supervised learning due to its simplicity and effectiveness \cite{CortesMR12,cristiani2002kta}. Our new formulation is based on unsupervised multiple kernel alignment criterion, inspired by existing supervised kernel learning. One can optimize this criterion by maximizing over both
$\boldsymbol{\gamma}$ and $\mathbf{H}$. Though theoretically elegant, we empirically observe that such
$\max_{\boldsymbol{\gamma}}\max_{\mathbf{H}}$ formulation does not achieve promising clustering performance, which is different from supervised kernel learning. We conjecture this is caused by the over-fitted optimization between $\boldsymbol{\gamma}$ and $\mathbf{H}$. On the other hand, from the optimization perspective of MKKM in Eq. (\ref{eq:MKKM_MinMin}), $\mathrm{Tr}\left(\mathbf{K}_{\boldsymbol{\gamma}}(\mathbf{I}-\mathbf{H}\mathbf{H}^{\top})\right)$ should be minimized. This objective can be decomposed into two terms, $\mathrm{Tr}\left(\mathbf{K}_{\boldsymbol{\gamma}}\right)$ and $-\mathrm{Tr}\left(\mathbf{K}_{\boldsymbol{\gamma}}\mathbf{H}\mathbf{H}^{\top}\right)$. The first term can be regarded as regularization on $\boldsymbol{\gamma}$, which should be optimized via minimizing $\boldsymbol{\gamma}$. The other one is the opposite of kernel alignment,  which should be minimized via maximizing $\mathbf{H}$. By taking both regularisation and partitioning into account, our SimpleMKKM proposes to optimize the kernel alignment criterion by minimizing $\boldsymbol{\gamma}$ and maximizing $\mathbf{H}$ as: %stated in Eq. (\ref{eq:MKKM_MinMax}).
\begin{equation}\label{eq:MKKM_MinMax}
\small{
\begin{split}
\min\nolimits_{\boldsymbol{\gamma}\in\Delta}\;\max\nolimits_{\mathbf{H}}\;&\;\mathrm{Tr}\left(\mathbf{K}_{\boldsymbol{\gamma}}\mathbf{H}\mathbf{H}^{\top}\right)\\
s.t.\;&\;\mathbf{H}\in\mathbb{R}^{n\times k},\,\mathbf{H}^{\top}\mathbf{H}=\mathbf{I}_{k},
\end{split}
}
\end{equation}
where $\Delta=\{\boldsymbol{\gamma}\in\mathbb{R}^{m}|\sum_{p=1}^{m}\gamma_{p}=1,\,\gamma_{p}\geq 0,\,\forall p\}$ and $\mathbf{K}_{\boldsymbol{\gamma}}=\sum_{p=1}^{m}\gamma_{p}^2\mathbf{K}_{p}$.

Though simple, the SimpleMKKM formulation in Eq. (\ref{eq:MKKM_MinMax}) has the following merits: (1) It is the first MKKM objective that, strictly coincides with the kernel alignment criterion via $\mathrm{Tr}\left(\mathbf{K}_{\boldsymbol{\gamma}}\mathbf{H}\mathbf{H}^{\top}\right)$ to tune kernel weights. In contrast, MKKM and its all variants adopt $\mathrm{Tr}\left(\mathbf{K}_{\boldsymbol{\gamma}}(\mathbf{I}-\mathbf{H}\mathbf{H}^{\top})\right)$ as the criterion by extending the objective of classic kernel \kmeans{} to multiple kernels. It is worth noting that the kernel alignment criterion is more general and can be used for any kernel tuning tasks. As a result, it can be used for multiple kernel clustering.
\cut{However, directly extending the objective of traditional kernel k-means to its multiple kernel case should be paid more attention. This is because the objective in single kernel may not be suitable anymore for multiple kernel.} %TH: This sentence seems super vague and just distracting from the main point.
(2) According to \cite{Bang2018robust}, regularisation by min-max optimization of $\boldsymbol{\gamma}$ and $\mathbf{H}$ generates more robust clusters by avoiding overfitting to noisy views or datapoints. (3) As we shall see next, while our formulation looks intractible, it actually leads to a more efficient and effective optimisation algorithm than the standard alternating strategies used for MKKM. Furthermore, unlike alternatives \cite{LiuDY0Z16,LiL0DYZ16} relying on regularisation by penalizing $\boldsymbol{\gamma}$, SimpleMKKM introduces no additional parameters beyond the number of clusters to form.
%our formulation is able to generate robust clustering partition by introducing min-max optimization which is helpful to avoid over-fitted optimization between variables. 3) Besides, our formulation is parameter-free once the number of clusters to form is specified.

Our new formulation in Eq. (\ref{eq:MKKM_MinMax}) cannot be readily solved by the widely adopted alternate optimization strategy, as done in MKKM and its variants. In the following, we design an efficient and effective reduced gradient descent algorithm. Firstly, we equivalently rewrite the optimization in Eq. (\ref{eq:MKKM_MinMax}) as,
\begin{equation}\label{eq:MKKM_MinMax1}
\small{
\min\nolimits_{\boldsymbol{\gamma}\in\Delta}\;\mathcal{J}(\boldsymbol{\gamma}),
}
\end{equation}
with
\begin{equation}\label{eq:MKKM_MinMax2}
\small{
\mathcal{J}(\boldsymbol{\gamma})=
\left\{\max\nolimits_{\mathbf{H}}\;\mathrm{Tr}\left(\mathbf{K}_{\boldsymbol{\gamma}}\mathbf{H}\mathbf{H}^{\top}
\right)\;\;s.t.\;\mathbf{H}^{\top}\mathbf{H}=\mathbf{I}_{k}\right\}.
}
\end{equation}
In this way, the min-max optimization is transformed to a minimization one, where its objective is a kernel \kmeans{}  optimal value function. In the following, we first prove the differentiability of $\mathcal{J}(\boldsymbol{\gamma})$, and apply the reduced gradient descent algorithm to decrease Eq. (\ref{eq:MKKM_MinMax1}).

\subsection{The Calculation of Reduced Gradient}

In the literature, several works discuss the existence and computation of derivatives of optimal value functions $\mathcal{J}(\boldsymbol{\gamma})$ \cite{BonnansS98,Chapelle2002,Rakotomamonjy08}. The most appropriate reference for our case is Theorem 4.1 in \cite{BonnansS98}, which has already been utilized to tune the hyper-parameters of SVM \cite{Chapelle2002} and optimize the kernel weights in multiple kernel learning \cite{Rakotomamonjy08}. The following Theorem \ref{theorem1} shows that $\mathcal{J}(\boldsymbol{\gamma})$ in Eq. (\ref{eq:MKKM_MinMax1}) is differentiable.
\begin{theorem}\label{theorem1}
$\mathcal{J}(\boldsymbol{\gamma})$ in Eq. (\ref{eq:MKKM_MinMax2}) is differentiable. Further, $\frac{\partial \mathcal{J}(\boldsymbol{\gamma})}{\partial \gamma_{p}}=2\gamma_{p}\mathrm{Tr}\left(\mathbf{K}_{p}\mathbf{H}^{\ast}{\mathbf{H}^{\ast}}^{\top}\right)$, where $\mathbf{H}^{\ast}=\left\{\mathrm{arg}\max_{\mathbf{H}}\;\mathrm{Tr}\left(\mathbf{K}_{\boldsymbol{\gamma}}\mathbf{H}\mathbf{H}^{\top}\right)\;s.t.\;\mathbf{H}^{\top}\mathbf{H}=\mathbf{I}_{k}\right\}$.
\end{theorem}
\begin{proof}
For any given $\boldsymbol{\gamma}\in\Delta$, the maximum of optimization problem $\max_{\mathbf{H}}\;\mathrm{Tr}\left(\mathbf{K}_{\boldsymbol{\gamma}}\mathbf{H}\mathbf{H}^{\top}\right)\;s.t.\;\mathbf{H}^{\top}\mathbf{H}=\mathbf{I}_{k}$ is uniqe, with $\tilde{\mathbf{H}}^{\ast}\in\{\tilde{\mathbf{H}}^{\ast}|\tilde{\mathbf{H}}^{\ast}=\mathbf{H}^{\ast}\mathbf{U},\,\mathbf{U}\mathbf{U}^{\top}=\mathbf{U}^{\top}\mathbf{U}=\mathbf{I}_{k}\}$ the corresponding maximizer. According to Theorem 4.1 in \cite{BonnansS98}, $\mathcal{J}(\boldsymbol{\gamma})$ in Eq. (\ref{eq:MKKM_MinMax2}) is differentiable, and $\frac{\partial \mathcal{J}(\boldsymbol{\gamma})}{\partial \gamma_{p}}=2\gamma_{p}\mathrm{Tr}(\mathbf{K}_{p}\tilde{\mathbf{H}}^{\ast}({\tilde{\mathbf{H}}^{\ast}})^{\top})=2\gamma_{p}\mathrm{Tr}(\mathbf{K}_{p}{\mathbf{H}}^{\ast}{\mathbf{H}^{\ast}}^{\top})$.
\end{proof}

\subsection{The Optimization Algorithm}

We propose to solve the optimization in Eq. (\ref{eq:MKKM_MinMax1}) with reduced gradient descent algorithms. We firstly calculate the gradient of $\mathcal{J}(\boldsymbol{\gamma})$ according to Theorem \ref{theorem1}, and then update $\boldsymbol{\gamma}$ with a descent direction by which the equality  and  non-negativity constraints on $\boldsymbol{\gamma}$ can be guaranteed.

To fulfill this goal, we firstly handle the equality constraint by computing the reduced gradient by following \citet{Rakotomamonjy08}. Let $\gamma_{u}$ be a non-zero component of $\boldsymbol{\gamma}$ and $\bigtriangledown\mathcal{J}(\boldsymbol{\gamma})$ denote the reduced gradient of $\mathcal{J}(\boldsymbol{\gamma})$. The $p$-th $(1\leq p\leq m)$ element of $\bigtriangledown\mathcal{J}(\boldsymbol{\gamma})$ is
\begin{equation}\label{eq:ReduceSubGradientP}
\small{
\left[\bigtriangledown\mathcal{J}(\boldsymbol{\gamma})\right]_{p}=\frac{\partial \mathcal{J}(\boldsymbol{\gamma})}{\partial \gamma_{p}}-\frac{\partial \mathcal{J}(\boldsymbol{\gamma})}{\partial \gamma_{u}}\;\;\forall\; p\neq u,
}
\end{equation}
and
\begin{equation}\label{eq:ReduceSubGradientU}
\small{
\left[\bigtriangledown\mathcal{J}(\boldsymbol{\gamma})\right]_{u}=\sum\nolimits_{p=1,p\neq u}^{m}\left(
\frac{\partial \mathcal{J}(\boldsymbol{\gamma})}{\partial \gamma_{u}}-\frac{\partial \mathcal{J}(\boldsymbol{\gamma})}{\partial \gamma_{p}}\right)
}
\end{equation}
Following the suggestion in  \citet{Rakotomamonjy08}, we choose $u$ to be the index of the largest component of vector $\boldsymbol{\gamma}$ which is considered to provide better numerical stability.

We then take the positivity constraints on $\boldsymbol{\gamma}$ into consideration in the descent direction. Note that $-\bigtriangledown\mathcal{J}(\boldsymbol{\gamma})$ is a descent direction since our aim is to minimize $\mathcal{J}(\boldsymbol{\gamma})$. However, directly using this direction would violate the positivity constraints in the case that if there is an index $p$ such that $\gamma_{p}=0$ and $\left[\bigtriangledown\mathcal{J}(\boldsymbol{\gamma})\right]_{p}>0$. In such case, the descent direction for that component should be set to $0$. This gives the descent direction for updating $\boldsymbol{\gamma}$ as
\begin{equation}\label{eq:DescentDirection}
\small{
	d_{p} =
	\begin{cases}
	0 & \text{if}\; \gamma_{p}=0 \; \text{and} \;\left[\bigtriangledown\mathcal{J}(\boldsymbol{\gamma})\right]_{p}>0\\
	-\left[\bigtriangledown\mathcal{J}(\boldsymbol{\gamma})\right]_{p} & \text{if} \; \gamma_{p}>0 \; \text{and} \; p\neq u\\
	- \left[\bigtriangledown\mathcal{J}(\boldsymbol{\gamma})\right]_{u} & \text{if}  \; p= u.
	\end{cases}
	}
\end{equation}
After a descent direction $\mathbf{d}=[d_{1},\cdots,d_{m}]^{\top}$ is computed by Eq. (\ref{eq:DescentDirection}), $\boldsymbol{\gamma}$ can be calculated via the updating scheme $\boldsymbol{\gamma}\leftarrow \boldsymbol{\gamma}+\alpha \mathbf{d}$, where $\alpha$ is the optimal step size. It can be selected by a one-dimensional line search strategy such as Armijo's rule.  The whole algorithm procedure solving the optimization problem in Eq.~(\ref{eq:MKKM_MinMax}) is outlined in Algorithm~\ref{SimpleMKKM}.

\begin{algorithm}[tb]
   \caption{SimpleMKKM}\label{SimpleMKKM}
\begin{algorithmic}[1]
   \STATE {\bfseries Input:} $\{\mathbf{K}_{p}\}_{p=1}^{m},\;k,\,t=1$.
   \STATE Initialize $\boldsymbol{\gamma}^{(1)} = \mathbf{1}/m,\,flag=1$.
   \WHILE{flag}
   \STATE compute $\mathbf{H}$ by solving a kernel k-means with $\mathbf{K}_{\boldsymbol{\gamma}^{(t)}}=\sum_{p=1}^{m}\big(\gamma_{p}^{(t)}\big)^{2}\mathbf{K}_{p}$.
   \STATE compute $\frac{\partial \mathcal{J}(\boldsymbol{\gamma})}{\partial \gamma_{p}}\,(p=1,\cdots,m)$ and the descent direction $\mathbf{d}^{(t)}$ in Eq. (\ref{eq:DescentDirection}).
   \STATE update $\boldsymbol{\gamma}^{(t+1)}\leftarrow \boldsymbol{\gamma}^{(t)}+\alpha \mathbf{d}^{(t)}$.
   \IF{$\max{|\boldsymbol{\gamma}^{(t)}-\boldsymbol{\gamma}^{(t-1)}|\leq1e-4}$}
   \STATE flag=0.
   \ENDIF
   \STATE $ t \leftarrow t+1$.
   \ENDWHILE
\end{algorithmic}
\end{algorithm}

\subsection{Computational Complexity and Convergence}

We discuss the computational complexity of SimpleMKKM. From Algorithm \ref{SimpleMKKM}, at each iteration, SimpleMKKM needs to solve a kernel \kmeans{} problem, calculate the reduced gradient, and search optimal step size. Therefore, its computational complexity at each iteration is $\mathcal{O}(n^3+m*n^3+m*n_{0})$, where $n_{0}$ is the maximal number of operations required to find the optimal step size. As observed, SimpleMKKM does not significantly increase the computational complexity of existing MKKM algorithms, as also validated by the experimental results in Figure \ref{FigRunningTime}.

We then briefly discuss the convergence of SimpleMKKM. Note that Eq. (\ref{eq:MKKM_MinMax2}) is a traditional kernel \kmeans{} which has a  global optimum. Under this condition, the gradient computation in Theorem~\ref{theorem1} is exact, and our algorithm performs reduced gradient descent on a continuously differentiable function $\mathcal{J}(\boldsymbol{\gamma})$ defined on the simplex $\{\boldsymbol{\gamma}\in\mathbb{R}^{m}|\sum_{p=1}^{m}\gamma_{p}=1,\,\gamma_{p}\geq 0,\,\forall p\}$, which does converge to the minimum of $\mathcal{J}(\boldsymbol{\gamma})$ \cite{Rakotomamonjy08}. The quick convergence of SimpleMKKM is validated by the experimental results in Figure~\ref{FigObjective}.

We conclude this section by discussing the differences with MKKM-MM \cite{Bang2018robust}. Though both works share a min-max (max-min) framework, their differences can be summarized from the following three aspects: (1) The objectives are different. SimpleMKKM adopts the unsupervised kernel alignment criterion while MKKM-MM inherits the objective of MKKM, which can be clearly seen from Eq.~(\ref{eq:MKKM_MM}) and Eq.~(\ref{eq:MKKM_MinMax}). Further,
MKKM-MM applies the $\ell_{2}$-norm constraints on $\boldsymbol{\gamma}$ to avoid sparse solutions. However, although using the $\ell_{1}$-norm constraint, our SimpleMKKM  still obtains non-sparse solution, as shown by the results in Figure \ref{FigKernelWeights}. (2) More importantly, the optimization strategies are totally different. MKKM-MM follows the widely used alternating optimization paradigm to solve Eq.~(\ref{eq:MKKM_MM}). In contrast, we, for the first time, reformulate the MKKM as a minimization problem, and develop a reduced gradient descent algorithm to efficiently solve it. (3) The clustering performance is different. We empirically compare their clustering performance, and observe that SimpleMKKM consistently and significantly outperforms MKKM-MM on all 11 benchmark datasets, as shown in Table~\ref{ClusteringACC}.

\section{The Generalization Analysis}

Generalization error for \kmeans{} clustering has been studied by fixing the centroids obtained in the training process and computing their generalization to testing data  \cite{maurer2010k,liu2016dimensionality}. In this section, we study how the centroids obtained by the proposed SimpleMKKM generalizes onto test data by deriving its generalization bound.

We now define the error of SimpleMKKM. Let $\hat{\mathbf{C}} = [\hat{\mathbf{C}}_1,\cdots,\hat{\mathbf{C}}_k]$ be the learned matrix composed of the $k$ centroids and $\hat{\boldsymbol{\gamma}}$ the learned kernel weights by the proposed SimpleMKKM, where $\hat{\mathbf{C}}_v=\frac{1}{|\hat{\mathbf{C}}_v|}\sum_{j\in\hat{\mathbf{C}}_v}\phi_{\hat{\boldsymbol{\gamma}}}(\mathbf{x}_{j}),1\leq c\leq k$. By defining $\Theta=\{\mathbf{e}_1,\cdots,\mathbf{e}_k\}$, effective SimpleMKKM clustering should make the following error small
\begin{equation}\label{eq:Generalization1}
\small{
1-\mathbb{E}_{\mathbf{x}}\left[\max\nolimits_{\mathbf{y}\in\Theta}
\langle \phi_{\hat{\boldsymbol{\gamma}}}(\mathbf{x}),\hat{\mathbf{C}}\mathbf{y}\rangle_{\mathcal{H}^{k}}\right],
}
\end{equation}
where $\phi_{\hat{\boldsymbol{\gamma}}}(\mathbf{x})=
[\hat{\boldsymbol{\gamma}}_{1}\phi_{1}^{\top}(\mathbf{x}),\cdot,\hat{\boldsymbol{\gamma}}_{m}\phi_{1}^{\top}(\mathbf{x})]^{\top}$ is the learned feature map associated with the kernel function $K_{\hat{\boldsymbol{\gamma}}}(\cdot,\cdot)$ and  $\mathbf{e}_1,\cdots,\mathbf{e}_k$ form the orthogonal bases of $\mathbb{R}^k$. Intuitively, it says the expected alignment between test points and their closest centroid should be high. We show how the proposed algorithm achieves this goal.

Let us define a function class first:
\begin{equation}\label{eq:Generalization2}
\small{
\begin{split}
\mathcal{F}= &\Big\{f:\;\mathbf{x}\mapsto 1 - \max\nolimits_{\mathbf{y}\in\Theta}\langle \phi_{{\boldsymbol{\gamma}}}(\mathbf{x}),\mathbf{C}\mathbf{y}\rangle_{\mathcal{H}^{k}} {\Big|}\boldsymbol{\gamma}^\top\mathbf{1}_m=1,\\
&\qquad\gamma_p\geq 0, \mathbf{C}\in\mathcal{H}^k,\,|K_p(\mathbf{x},\tilde{\mathbf{x}})|\leq b,\,\forall p, \forall \mathbf{x}\in\mathcal{X}\Big\},
\end{split}
}
\end{equation}
where $\mathcal{H}^k$ stands for the multiple kernel Hilbert space.

\begin{theorem}\label{maintl}
For any $\delta>0$, with probability at least $1-\delta$, the following holds for all $f\in \mathcal{F}$:
\begin{equation}\label{eq:Generalization3}
\small{
\begin{split}
\mathbb{E}\left[f(\mathbf{x})\right]&\leq \frac{1}{n}\sum\nolimits_{i=1}^{n}f(\mathbf{x}_i)+\frac{\sqrt{\pi/2}bk}{\sqrt{n}} + (1+b)\sqrt{\frac{\log{1/\delta}}{2n}}.
\end{split}
}
\end{equation}
%where
%\begin{equation}\label{eq:Generalization4}
%\small{
%\mathcal{G}_{1n}(\boldsymbol{\gamma}, t)  \triangleq\mathbb{E}_\beta\left[\sup_{\boldsymbol{\gamma},t}\sum_{i=1}^n\sum_{p,q=1}^{m}\beta_{ipq}\left<\gamma_p t({\bf x}_i^{(p)}),\gamma_q t(\mathbf{x}_i^{(q)})\right>\right],
%}
%\end{equation}
%\begin{equation}\label{eq:Generalization5}
%\small{
%\mathcal{G}_{2n}(\boldsymbol{\gamma}, t)=\mathbb{E}_\beta\left[ \sup_{\boldsymbol{\gamma},t}\sum_{i=1}^n\sum_{c=1}^k\sum_{p=1}^{m}\beta_{icp}\gamma_{p}t({\bf x}_i^{(p)})\right],
%}
%\end{equation}
%and $\beta_{ipq}, \beta_{icp}, i\in\{1,\cdots,n\}, p,q\in\{1,\cdots,m\}, c\in\{1,\cdots,k\}$ are i.i.d. Gaussian random variables with zero mean and unit standard deviation.
\end{theorem}
The detailed proof is provided in the appendix due to space limit.

According to Theorem \ref{maintl}, for any learned $\hat{\boldsymbol{\gamma}}$ and $\hat{\mathbf{C}}$, to achieve a small
\begin{equation}\label{eq:Generalization6}
\small{
\begin{split}
\mathbb{E}_{\mathbf{x}}[f(\mathbf{x})] = 1 - \mathbb{E}_{\mathbf{x}}\left[\max\nolimits_{\mathbf{y}\in\Theta}\left\langle \phi_{\hat{\boldsymbol{\gamma}}}(\mathbf{x}),\hat{\mathbf{C}}\mathbf{y}\right\rangle_{\mathcal{H}^{k}}\right],
\end{split}
}
\end{equation}
the corresponding $\frac{1}{n}\sum_{i=1}^{n}f(\mathbf{x}_i)$ needs to be as small as possible. Assume that ${\boldsymbol{\gamma}}$ and ${\mathbf{C}}$ are obtained by minimizing $\frac{1}{n}\sum_{i}^{n}f(\mathbf{x}_i)$ and that ${\mathbf{H}}$ is constrained to be orthogonal, we have
\begin{equation}\label{eq:Generalization7}
\small{
\begin{split}
% \frac{1}{n^2}\sum\nolimits_{i,j=1}^{n}f(\mathbf{x}_i,\mathbf{x}_j) &\leq 1 - \frac{1}{n^2}\max_{\mathbf{H}}\mathrm{Tr}(\mathbf{K}_{\boldsymbol{\gamma}}\mathbf{H}\mathbf{H}^\top)\\
\frac{1}{n}\sum\nolimits_{i=1}^{n}f(\mathbf{x}_i) &\leq 1 - \frac{1}{n}\mathrm{Tr}(\mathbf{K}_{\boldsymbol{\gamma}}\mathbf{H}\mathbf{H}^\top)\\
% &\leq 1 - \frac{1}{n^2}\min_{\boldsymbol{\gamma}}\mathrm{Tr}(\mathbf{K}_{\boldsymbol{\gamma}}\mathbf{H}\mathbf{H}^\top)
\end{split}
}
\end{equation}
because the proposed algorithm poses a constraint $\mathbf{H}^\top\mathbf{H}=\mathbf{I}_k$ which will make the corresponding centroids non-optimal for minimizing $\frac{1}{n}\sum\nolimits_{i=1}^{n}f(\mathbf{x}_i)$.
This means that $1 - \frac{1}{n}\mathrm{Tr}(\mathbf{K}_{\boldsymbol{\gamma}}\mathbf{H}\mathbf{H}^\top)$ is an upper bound of $\frac{1}{n}\sum\nolimits_{i=1}^{n}f(\mathbf{x}_i)$. To minimize the upper bound, we may have to maximize over $\boldsymbol{\gamma}$ and $\mathbf{H}$, leading to $\max_{\boldsymbol{\gamma}}\max_{\mathbf{H}}\mathrm{Tr}(\mathbf{K}_{\boldsymbol{\gamma}}\mathbf{H}\mathbf{H}^\top)$. However, it is  intractable to find a good solution to $\boldsymbol{\gamma}$ and $\mathbf{H}$ under this criterion, and it is prone to over-fitted solutions \cite{Bang2018robust}. Instead, we take one of its lower bounds, $\min_{\boldsymbol{\gamma}}\max_{\mathbf{H}}\mathrm{Tr}(\mathbf{K}_{\boldsymbol{\gamma}}\mathbf{H}\mathbf{H}^\top)$ as the the objective of SimpleMKKM in Eq. (\ref{eq:MKKM_MinMax}). %, which is the objective of our proposed SimpleMKKM.
%It is worth point out that $1-\min_{\boldsymbol{\gamma}}\max_{\mathbf{H}}\mathrm{Tr}(\mathbf{K}_{\boldsymbol{\gamma}}\mathbf{H}\mathbf{H}^\top)$ is still an upper bound of $\frac{1}{n^2}\sum\nolimits_{i,j=1}^{n}f(\mathbf{x}_i,\mathbf{x}_j)$.
This analysis verifies the good generalization ability of the proposed SimpleMKKM.
%The detailed proof are provided in the supplemental material due to space limit.

\section{Experimental Results}

In this section, we conduct a comprehensive experimental study to evaluate the proposed SimpleMKKM in terms of clustering performance, the learned kernel weights, the running time, and convergence.

\subsection{Experimental Settings}

A number of standard MKKM benchmark datasets are adopted to evaluate  SimpleMKKM, including \emph{Flo17}\footnote{\url{www.robots.ox.ac.uk/~vgg/data/flowers/17/}}, \emph{Flo102}\footnote{\url{www.robots.ox.ac.uk/~vgg/data/flowers/102/}}, \emph{PFold}\footnote{\url{mkl.ucsd.edu/dataset/protein-fold-prediction}}, \emph{CCV}\footnote{\url{www.ee.columbia.edu/ln/dvmm/CCV/}}, \emph{Digit}\footnote{\url{http://ss.sysu.edu.cn/py/}}, \emph{Cal}\footnote{\url{www.vision.caltech.edu/Image_Datasets/Caltech101/}}.
% For \emph{Cal}, the first $5$, $10$, $15$, $20$, $25$ and $30$ classes are selected to construct six sub-datasets
Meanwhile, six sub-datasets, i.e. \emph{Cal-5}, \emph{Cal-10}, \emph{Cal-15}, \emph{Cal-20}, \emph{Cal-25} and \emph{Cal-30}, are constructed via selecting the first $5$, $10$, $15$, $20$, $25$ and $30$ classes respectively from the \emph{Cal} data. Their details are shown in Table~\ref{tab:datasets}.  It can be observed that the number of samples, kernels and categories of these datasets shows considerable variation, providing a good platform to compare the performance of different clustering algorithms.

\begin{table}[t]
\vspace{-6pt}
\linespread{1.0}
\centering
\caption{Specification of  our 11 benchmark datasets.}\label{tab:datasets}
\footnotesize
\begin{tabular}{lrrr} \hline
\multirow{2}{*}{Dataset} & \multicolumn{3}{c}{Number of} \\ %\cline{2-4}
 & \multicolumn{1}{c}{Samples} & \multicolumn{1}{c}{Kernels} & \multicolumn{1}{c}{Clusters} \\ \hline
Flo17 & 1360 & 7 & 17 \\
Flo102 & 8189 & 4 & 102 \\
PFold & 694 & 12 & 27 \\
CCV & 6773 & 3 & 20 \\
Digit & 2000 & 3 & 10 \\
Cal-5 & 75 & 25 & 5 \\
Cal-10 & 150 & 25 & 10 \\
Cal-15 & 225 & 25 & 15 \\
Cal-20 & 300 & 25 & 20 \\
Cal-25 & 375 & 25 & 25 \\
Cal-30 & 450 & 25 & 30 \\ \hline
% Caltech101 & 1530 & 25 & 102 \\ \hline
\end{tabular}
\vspace{-8pt}
\end{table}

\begin{table*}[!htbp]
\centering
\vspace{-29pt}
\caption{Empirical evaluation and comparison of SimpleMKKM with eight baseline methods on eleven datasets in terms of clustering accuracy (ACC), normulaized mutual information (NMI) and Purity.
Boldface means no statistical difference from the best one.}\label{ClusteringACC}
\vspace{-15pt}
\begin{center}
\begin{small}
\begin{sc}
\resizebox{1.0\linewidth}{!}{
\begin{tabular}{lccccccccc}
\toprule
\scriptsize{Dataset}        & \scriptsize{Avg-KKM}        & \scriptsize{MKKM}          & \scriptsize{LMKKM}         & \scriptsize{ONKC}          & \scriptsize{MKKM-MiR}       & \scriptsize{LKAM}          & \scriptsize{LF-MVC}        &  \scriptsize{MKKM-MM} & \scriptsize{SimpleMKKM}\\
\toprule
\multicolumn{10}{c}{ACC}\\
\midrule
Flo17       & 51.3$\pm$ 1.4  & 43.9$\pm$ 1.7 & 42.7$\pm$ 1.7 & 43.4$\pm$ 2.0 &  \textbf{57.7$\pm$ 1.2} & 48.9$\pm$ 0.9 & 56.7$\pm$ 1.5 & 48.5$\pm$1.9 & \textbf{58.9$\pm$ 1.3} \\
Flo102      & 26.8$\pm$ 0.8  & 22.5$\pm$ 0.5 &     -     & 39.2$\pm$ 0.8 &  39.0$\pm$ 1.2 & 40.4$\pm$ 0.9 & 29.2$\pm$ 0.9 & 24.1$\pm$0.6 & \textbf{42.7$\pm$ 1.2} \\
PFold    & 29.1$\pm$ 1.4  & 27.1$\pm$ 1.0 & 22.4$\pm$ 0.7 & \textbf{35.4$\pm$ 1.5} &  \textbf{34.3$\pm$ 1.6} & \textbf{34.2$\pm$ 1.6} & 32.0$\pm$ 1.6 & 29.0$\pm$1.5 & \textbf{34.4$\pm$ 1.9} \\
CCV            & 19.6$\pm$ 0.6  & 18.0$\pm$ 0.5 & 18.6$\pm$ 0.2 & {22.1$\pm$ 0.6} &  20.8$\pm$ 0.8 & 19.0$\pm$ 0.3 & \textbf{23.1$\pm$ 0.5} & 18.8$\pm$0.7 & {22.1$\pm$ 0.7} \\
Digit    & 88.8$\pm$ 0.1  & 47.2$\pm$ 0.6 & 47.2$\pm$ 0.7 & 89.5$\pm$ 0.1 &  87.4$\pm$ 0.1 & \textbf{95.0$\pm$ 0.3} & 89.2$\pm$ 0.1 & 47.3$\pm$0.7 & 90.3$\pm$ 0.1 \\
Cal-5   & \textbf{35.9$\pm$ 1.1}  & 27.5$\pm$ 0.8 & 28.1$\pm$ 0.7 & 35.1$\pm$ 1.3 &  35.7$\pm$ 1.1 & 33.8$\pm$ 1.0 & \textbf{37.1$\pm$ 1.3} & $33.9\pm1.3$ & \textbf{36.1$\pm$ 1.0} \\
Cal-10  & 31.8$\pm$ 0.9  & 22.1$\pm$ 0.5 & 21.4$\pm$ 0.7 & 30.9$\pm$ 0.9 &  32.5$\pm$ 1.1 & 30.3$\pm$ 0.8 & \textbf{33.8$\pm$ 1.2} & 30.1$\pm$0.9 & \textbf{33.4$\pm$ 1.2} \\
Cal-15  & 30.5$\pm$ 0.9  & 19.7$\pm$ 0.5 & 	  -     & 28.7$\pm$ 0.8 &  30.7$\pm$ 1.0 & 28.5$\pm$ 0.8 & \textbf{32.6$\pm$ 0.9} & 28.2$\pm$1.1 & \textbf{31.8$\pm$ 0.9} \\
Cal-20  & 29.5$\pm$ 1.1  & 18.3$\pm$ 0.5 &     -    & 27.7$\pm$ 0.9 &  29.4$\pm$ 1.1 & 27.4$\pm$ 0.6 & \textbf{31.7$\pm$ 0.9} & 27.0$\pm$0.9 & \textbf{31.4$\pm$ 0.9} \\
Cal-25  & 29.2$\pm$ 0.9  & 17.0$\pm$ 0.5 &     -     & 26.4$\pm$ 0.9 &  27.5$\pm$ 1.0 & 25.5$\pm$ 0.8 & \textbf{31.1$\pm$ 0.9} & 26.2$\pm$0.7 & {30.0$\pm$ 0.9} \\
Cal-30  & 28.5$\pm$ 0.8  & 16.5$\pm$ 0.4 &    -    & 25.5$\pm$ 0.8 &  27.1$\pm$ 1.0 & 24.4$\pm$ 0.6 & \textbf{31.0$\pm$ 0.9} & 25.4$\pm$0.8 & \textbf{30.7$\pm$ 0.8} \\
\toprule
\multicolumn{10}{c}{NMI}\\
\midrule
Flo17       & 49.9$\pm$ 0.9  & 44.6$\pm$ 1.3 & 43.8$\pm$ 1.1 & 42.9$\pm$ 1.3 &  56.1$\pm$ 0.7 & 48.1$\pm$ 0.6 & 54.6$\pm$ 1.0 & 47.4$\pm$1.3 & \textbf{57.3$\pm$ 0.8} \\
Flo102      & 45.9$\pm$ 0.4  & 42.7$\pm$ 0.2 &     -    & 55.7$\pm$ 0.4 &  55.8$\pm$ 0.6 & 55.8$\pm$ 0.4 & 47.5$\pm$ 0.4 & 43.8$\pm$0.3 & \textbf{58.7$\pm$ 0.5}\\
PFold    & 40.3$\pm$ 1.2  & 38.1$\pm$ 0.6 & 34.8$\pm$ 0.6 & \textbf{44.1$\pm$ 0.8} &  \textbf{43.2$\pm$ 1.1} & \textbf{43.7$\pm$ 1.0} & 42.0$\pm$ 1.2 & 39.8$\pm$0.9 & \textbf{44.2$\pm$ 1.2} \\
CCV            & 16.8$\pm$ 0.3  & 15.1$\pm$ 0.5 & 14.4$\pm$ 0.1 & 18.4$\pm$ 0.3 &  17.9$\pm$ 0.4 & 16.8$\pm$ 0.2 & \textbf{18.9$\pm$ 0.3} & 15.7$\pm$0.5 & 18.2$\pm$ 0.3 \\
Digit    & 80.7$\pm$ 0.2  & 48.7$\pm$ 0.7 & 48.7$\pm$ 0.6 & 81.7$\pm$ 0.1 &  79.5$\pm$ 0.1 & \textbf{89.4$\pm$ 0.1} & 81.2$\pm$ 0.2 & 48.7$\pm$0.6 & 83.3$\pm$ 0.1 \\
Cal-5   & \textbf{70.3$\pm$ 0.6}  & 65.9$\pm$ 0.4 & 66.5$\pm$ 0.3 & 69.8$\pm$ 0.6 &  \textbf{70.2$\pm$ 0.5} & 68.4$\pm$ 0.5 & \textbf{70.8$\pm$ 0.6} & 69.3$\pm$0.7 & \textbf{70.3$\pm$ 0.4} \\
Cal-10  & 61.8$\pm$ 0.5  & 55.4$\pm$ 0.4 & 55.2$\pm$ 0.3 & 61.0$\pm$ 0.5 &  62.1$\pm$ 0.5 & 60.5$\pm$ 0.6 & \textbf{62.9$\pm$ 0.6} & 60.6$\pm$0.5 & \textbf{62.6$\pm$ 0.6} \\
Calt-15  & 57.2$\pm$ 0.5  & 49.3$\pm$ 0.5 &     -     & 55.8$\pm$ 0.4 &  57.2$\pm$ 0.6 & 56.0$\pm$ 0.5 & \textbf{58.7$\pm$ 0.5} & 55.6$\pm$0.5 & \textbf{58.2$\pm$ 0.5} \\
Cal-20  & 54.2$\pm$ 0.6  & 45.4$\pm$ 0.3 &     -     & 52.7$\pm$ 0.5 &  54.0$\pm$ 0.5 & 52.6$\pm$ 0.4 & \textbf{55.8$\pm$ 0.6} & 52.3$\pm$0.5 & \textbf{55.5$\pm$ 0.5} \\
Cal-25  & 52.1$\pm$ 0.6  & 42.3$\pm$ 0.4 &     -     & 49.9$\pm$ 0.6 &  51.1$\pm$ 0.6 & 49.7$\pm$ 0.4 & \textbf{53.6$\pm$ 0.5} & 49.7$\pm$0.4 & 52.9$\pm$ 0.5 \\
Cal-30  & 50.0$\pm$ 0.6  & 40.1$\pm$ 0.3 &     -     & 47.6$\pm$ 0.5 &  49.0$\pm$ 0.5 & 47.4$\pm$ 0.4 & \textbf{52.1$\pm$ 0.5} & 47.7$\pm$0.4 & \textbf{51.8$\pm$ 0.5} \\
\toprule
\multicolumn{10}{c}{Purity}\\
\midrule
Flo17       & 52.3$\pm$ 1.2  & 45.3$\pm$ 1.5 & 44.6$\pm$ 1.5 & 45.1$\pm$ 1.8 &  59.2$\pm$ 1.1 & 50.1$\pm$ 0.6 & 57.5$\pm$ 1.6 & 49.5$\pm$1.8 & \textbf{60.2$\pm$ 1.4} \\
Flo102      & 32.2$\pm$ 0.6  & 27.8$\pm$ 0.4 &     -     & 45.1$\pm$ 0.8 &  45.1$\pm$ 1.0 & 46.7$\pm$ 0.6 & 34.6$\pm$ 0.6 & 29.3$\pm$0.5 & \textbf{48.7$\pm$ 0.8}\\
PFold    & 37.3$\pm$ 1.6  & 33.7$\pm$ 0.9 & 31.1$\pm$ 1.0 & \textbf{42.0$\pm$ 1.2} &  \textbf{41.2$\pm$ 1.4} & \textbf{41.6$\pm$ 1.3} & 38.7$\pm$ 1.4 & 36.3$\pm$1.1 & \textbf{41.4$\pm$ 1.6} \\
CCV            & 23.7$\pm$ 0.4  & 22.3$\pm$ 0.5 & 22.0$\pm$ 0.2 & 24.4$\pm$ 0.5 &  23.4$\pm$ 0.7 & 22.2$\pm$ 0.3 & \textbf{25.8$\pm$ 0.4} & 23.1$\pm$0.7 & 25.2$\pm$ 0.6 \\
Digit    & 88.8$\pm$ 0.1  & 50.1$\pm$ 0.7 & 50.0$\pm$ 0.7 & 89.5$\pm$ 0.1 &  87.4$\pm$ 0.1 & \textbf{95.0$\pm$ 0.3} & 89.2$\pm$ 0.1 & 50.0$\pm$0.8 & 90.3$\pm$ 0.1 \\
Cal-5   & 37.3$\pm$ 1.2  & 28.3$\pm$ 0.8 & 28.6$\pm$ 0.7 & 36.4$\pm$ 1.3 &  \textbf{37.4$\pm$ 1.1} & 35.6$\pm$ 1.1 & \textbf{38.6$\pm$ 1.2} & 35.2$\pm$1.4 & \textbf{37.5$\pm$ 1.2} \\
Cal-10  & 33.6$\pm$ 0.9  & 23.5$\pm$ 0.5 & 22.8$\pm$ 0.6 & 32.9$\pm$ 0.9 &  \textbf{34.7$\pm$ 1.0} & 32.2$\pm$ 0.8 & \textbf{35.8$\pm$ 1.2 }& 31.9$\pm$0.9 & \textbf{35.3$\pm$ 1.1} \\
Cal-15  & 32.3$\pm$ 0.8  & 21.2$\pm$ 0.5 &     -     & 30.6$\pm$ 0.8 &  32.5$\pm$ 1.0 & 30.1$\pm$ 0.7 & \textbf{34.4$\pm$ 0.9 }& 30.0$\pm$0.9 & \textbf{33.7$\pm$ 0.8} \\
Cal-20  & 31.5$\pm$ 1.0  & 19.9$\pm$ 0.5 &     -     & 29.7$\pm$ 0.7 &  31.4$\pm$ 1.0 & 29.2$\pm$ 0.7 & \textbf{33.7$\pm$ 0.9} & 28.9$\pm$0.8 & \textbf{33.4$\pm$ 0.7} \\
Cal-25  & 31.1$\pm$ 0.7  & 18.7$\pm$ 0.4 &     -    & 28.4$\pm$ 0.8 &  29.6$\pm$ 0.8 & 27.7$\pm$ 0.8 & \textbf{33.5$\pm$ 0.8} & 28.1$\pm$0.8 & 32.2$\pm$ 0.9 \\
Cal-30  & 30.5$\pm$ 0.8  & 18.0$\pm$ 0.4 &     -    & 27.4$\pm$ 0.7 &  29.1$\pm$ 0.9 & 26.4$\pm$ 0.5 & \textbf{33.0$\pm$ 0.9} & 27.3$\pm$0.7 & \textbf{32.7$\pm$ 0.7} \\
\bottomrule
\end{tabular}}
\end{sc}
\end{small}
\end{center}
\vspace{-25pt}
\end{table*}

For all data sets, the true number of clusters $k$ is assumed known and is set as the true number of classes. The widely used clustering accuracy (ACC), normalized mutual information (NMI) and purity are applied to evaluate the clustering performance.

\begin{figure*}[!htbp]
\vspace{-20pt}
\centering
\subfigure{\includegraphics[width=0.31\textwidth]{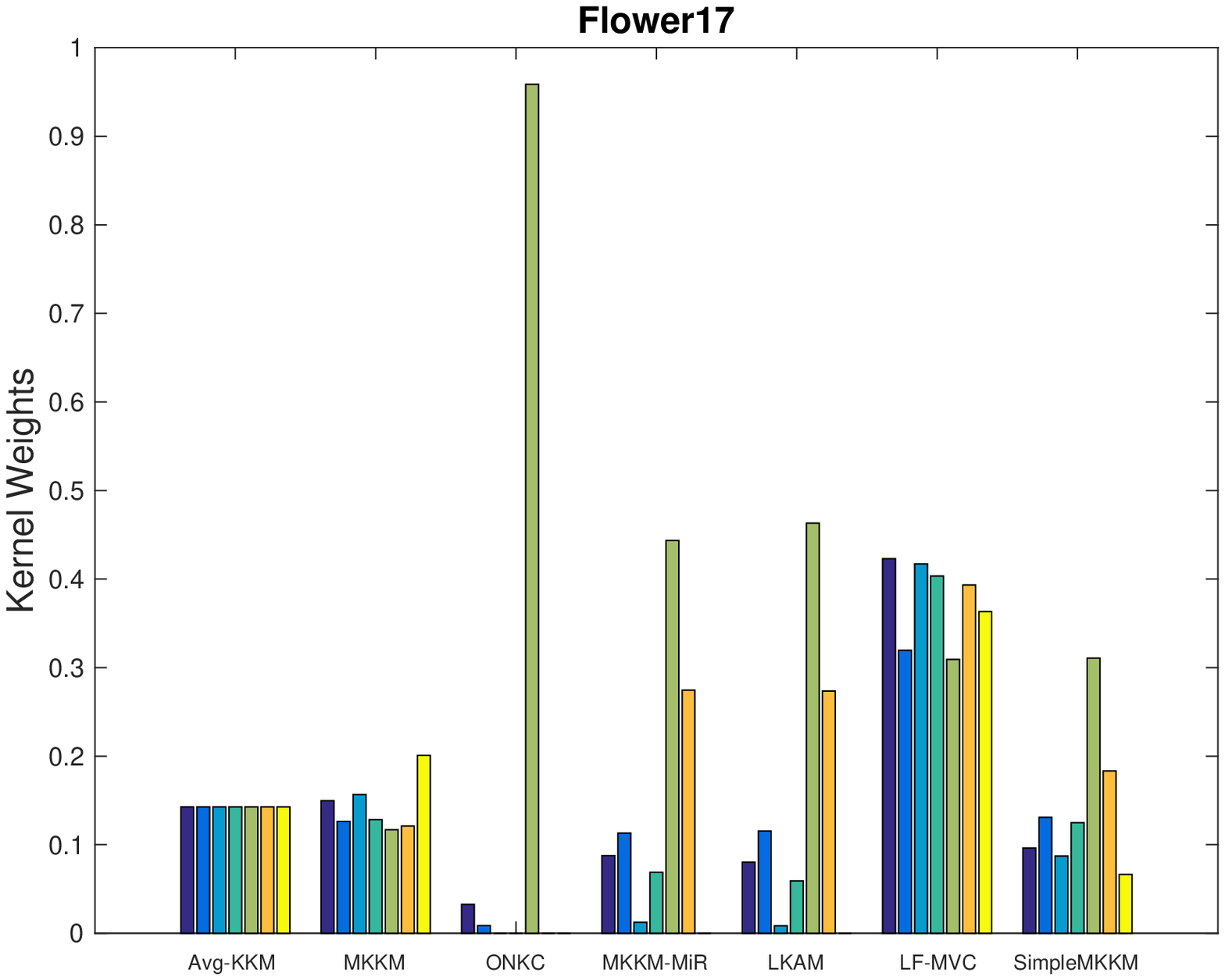}\label{KernelWeights_Flower17}}%
\subfigure{\includegraphics[width=0.31\textwidth]{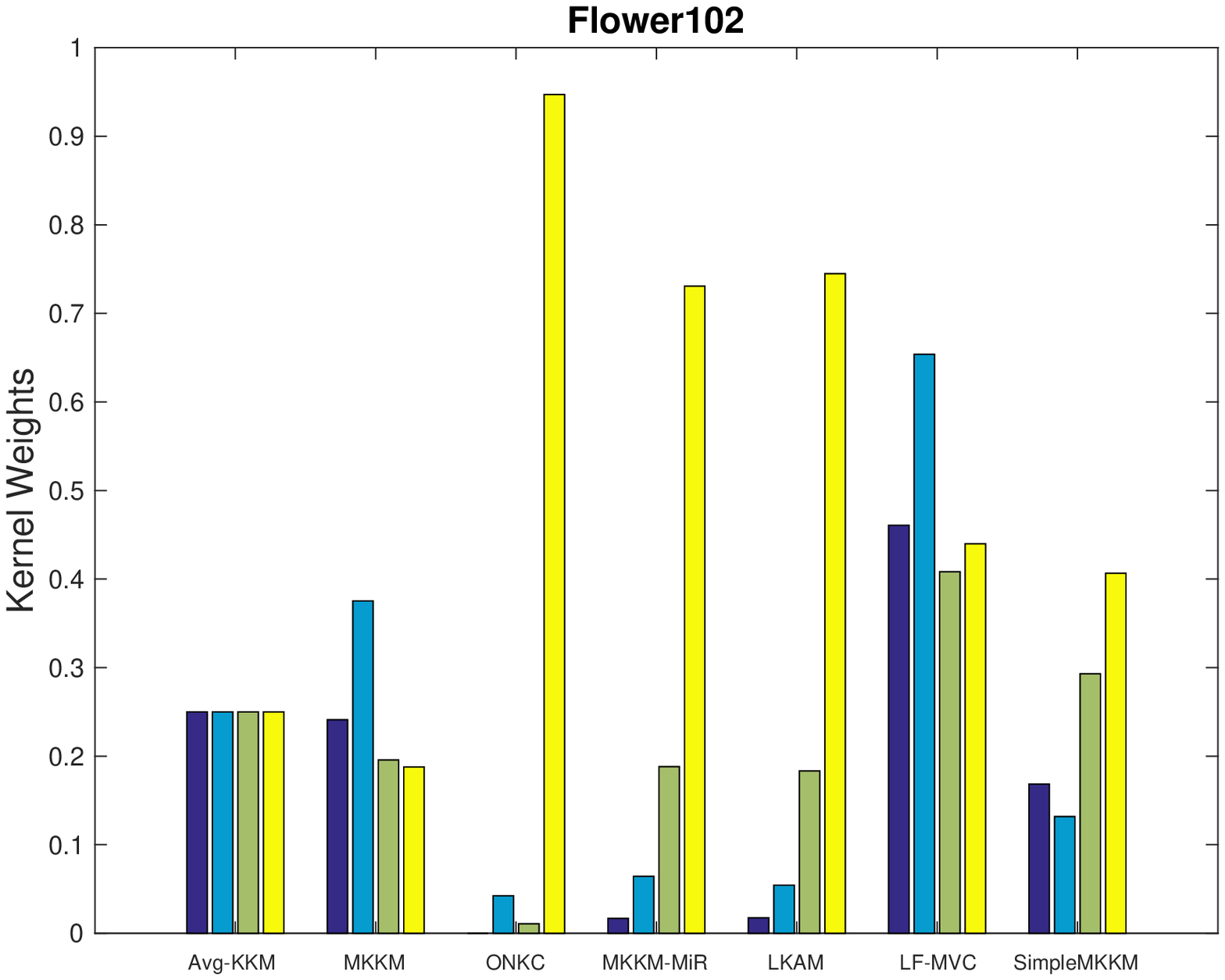}\label{KernelWeights_Flower102}}%
\subfigure{\includegraphics[width=0.31\textwidth]{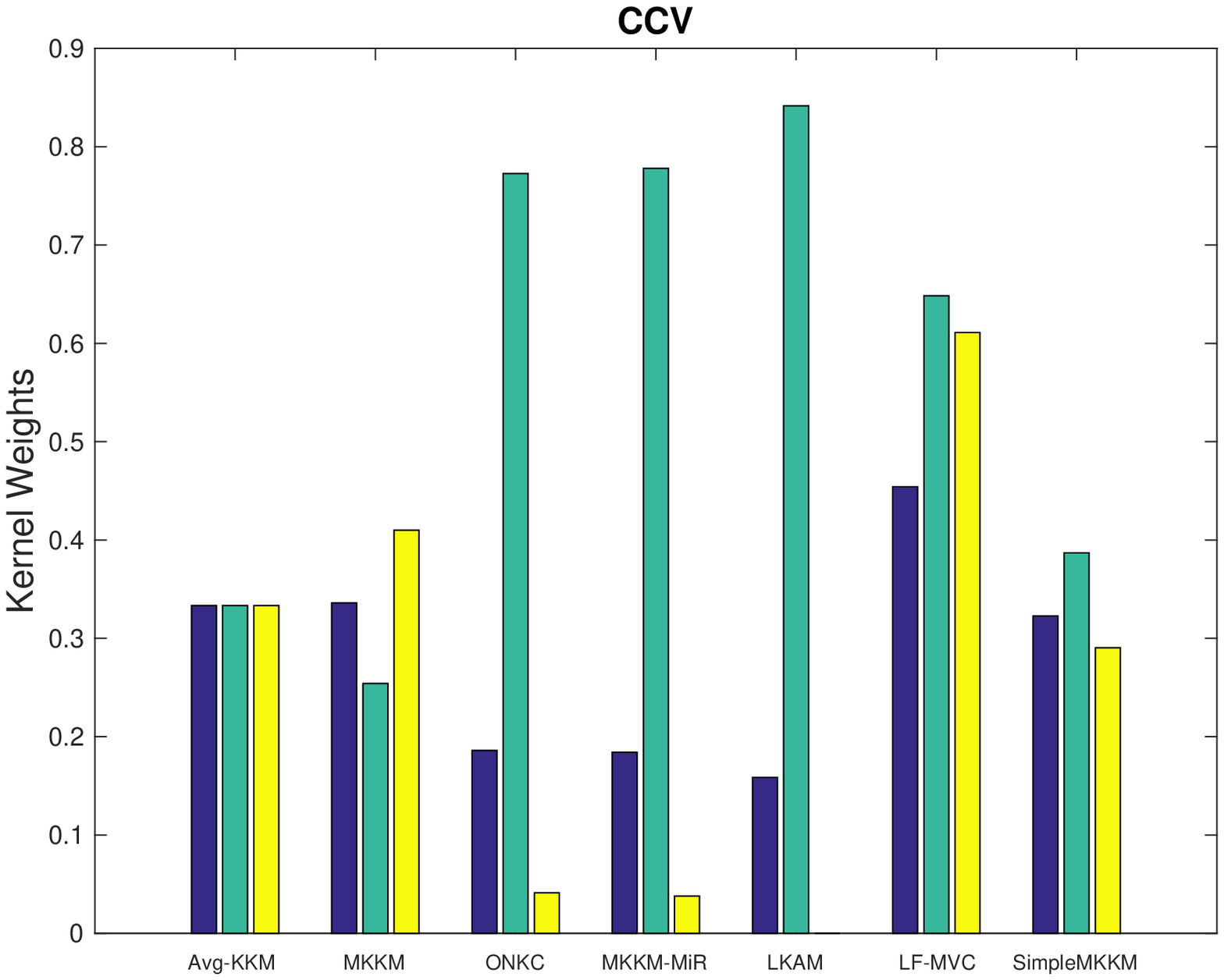}\label{KernelWeights_CCV}}\\[-15pt]
%\subfigure{\includegraphics[width=0.25\textwidth]{myFigures/KernelWeights_Caltech102-5.eps}\label{flower17_accIter}}%
%\subfigure{\includegraphics[width=0.25\textwidth]{myFigures/KernelWeights_Caltech102-10.eps}\label{Flower102_accIter}}%
%\subfigure{\includegraphics[width=0.25\textwidth]{myFigures/KernelWeights_Caltech102-15.eps}\label{UCIdigtal_accIter}}%
%\subfigure{\includegraphics[width=0.25\textwidth]{myFigures/KernelWeights_Caltech102-20.eps}\label{CCV_accIter}}\\
%\subfigure{\includegraphics[width=0.25\textwidth]{myFigures/KernelWeights_Caltech102-25.eps}\label{Caltech102_accIter}}%
\subfigure{\includegraphics[width=0.31\textwidth]{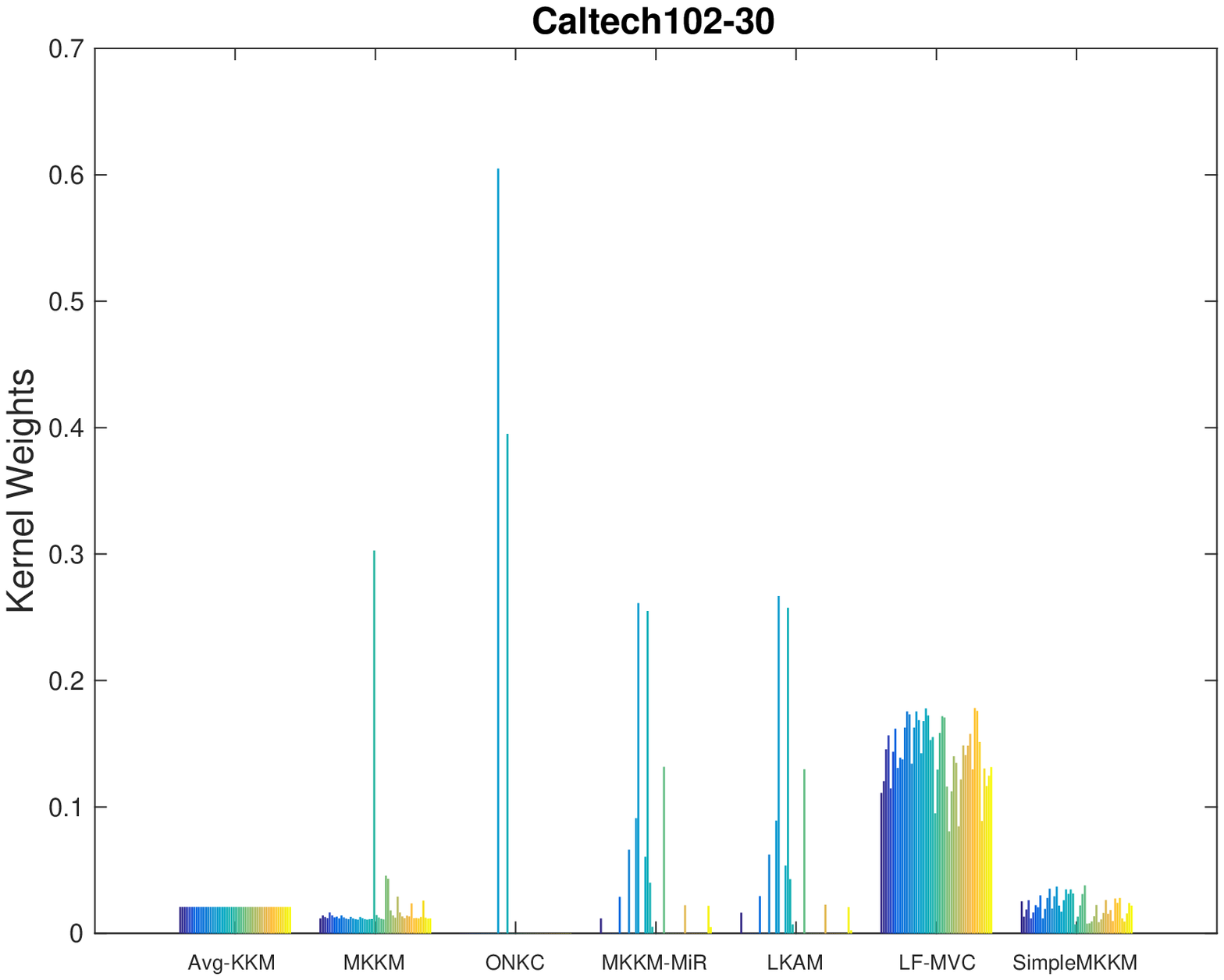}\label{KernelWeights_Caltech102-30}}%
\subfigure{\includegraphics[width=0.31\textwidth]{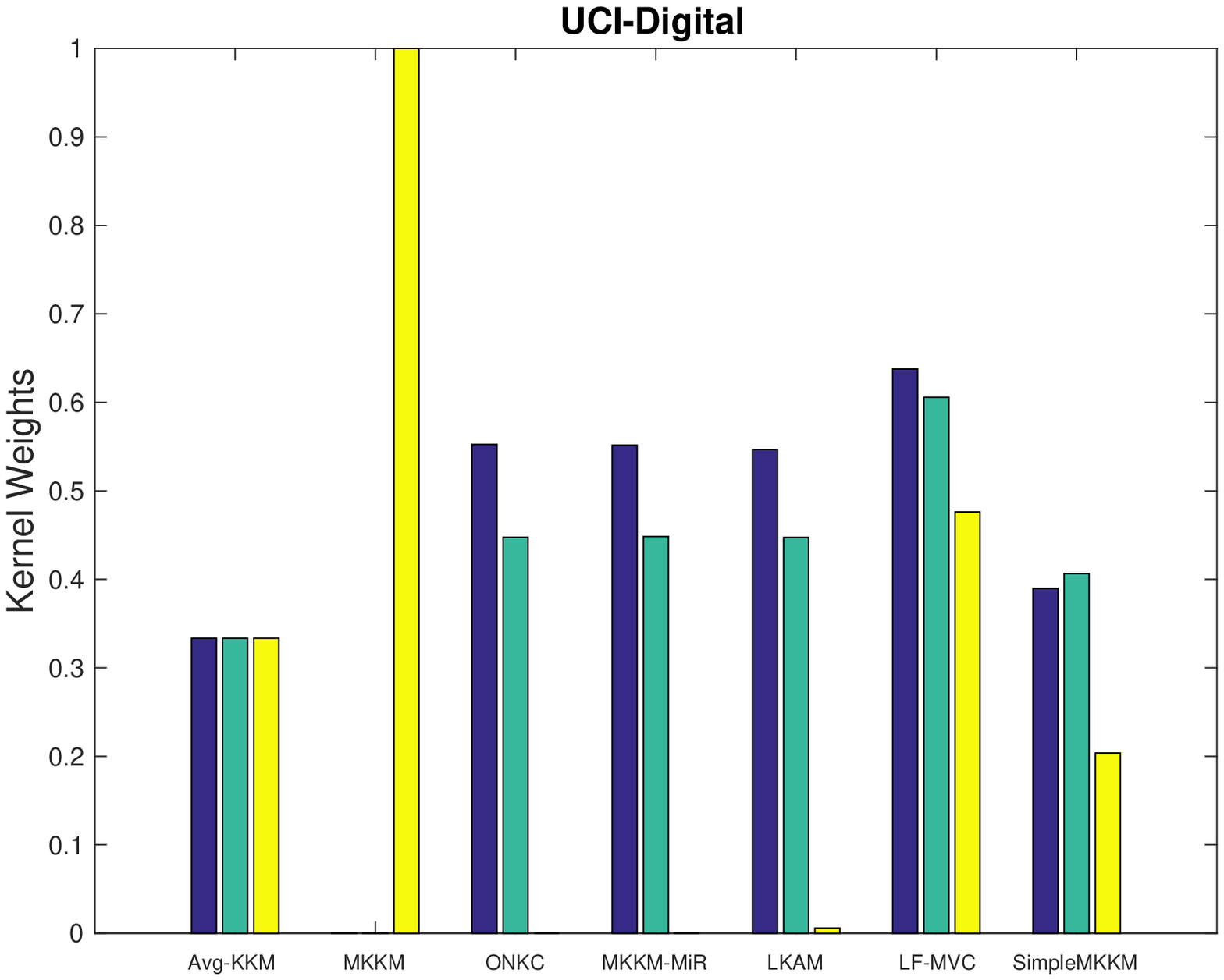}\label{KernelWeights_UCIdigital}}%
\subfigure{\includegraphics[width=0.31\textwidth]{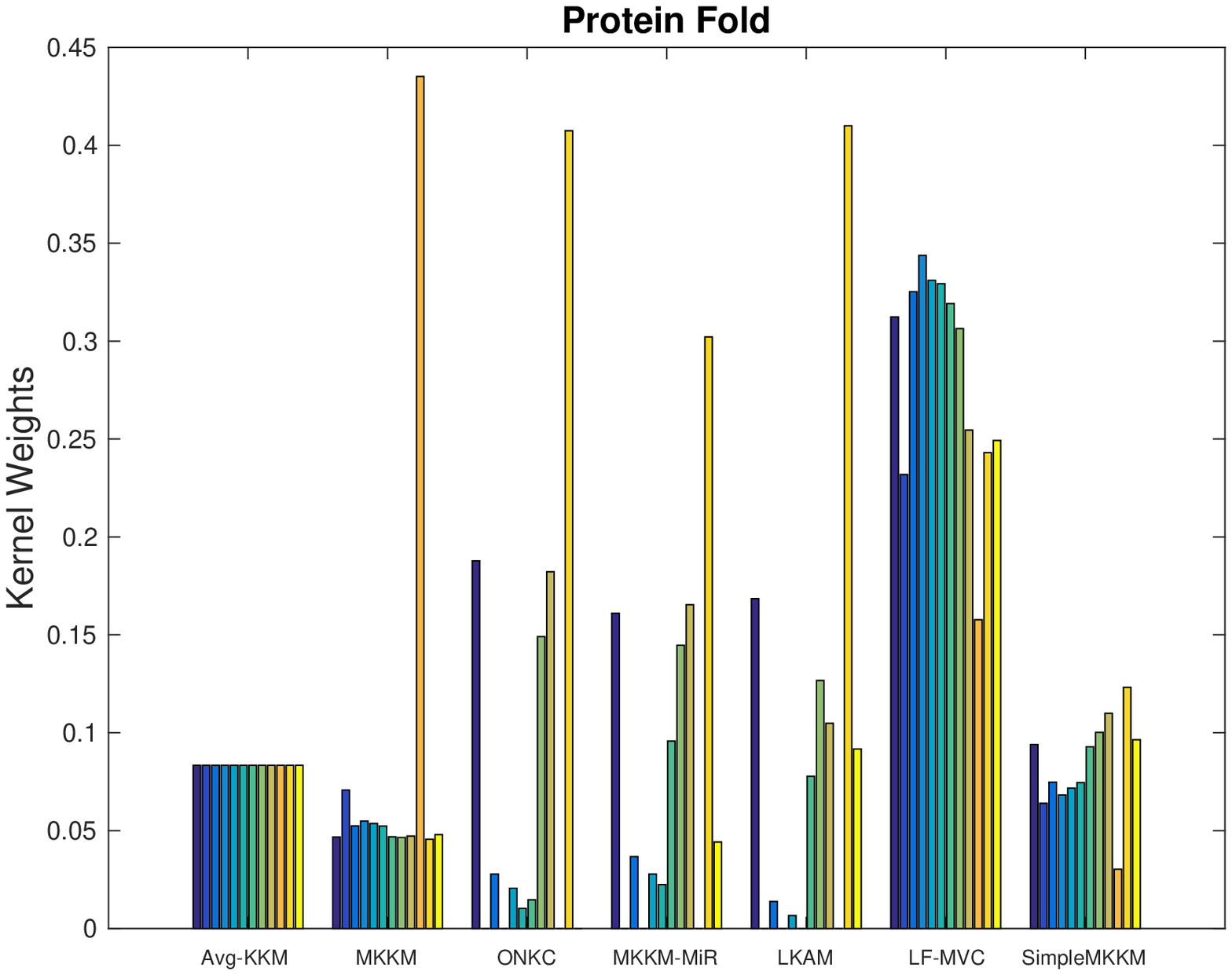}\label{KernelWeights_ProteinFold}}\\
\vspace{-20pt}
\caption{The kernel weights learned by different algorithms. SimpleMKKM maintains reduced sparsity compared to several competitors. Other datasets omitted due to space limit.}\label{FigKernelWeights}
\vspace{-30pt}
\end{figure*}

\begin{figure*}[!htbp]
\vspace{-14pt}
\centering
\subfigure{\includegraphics[width=0.33\textwidth]{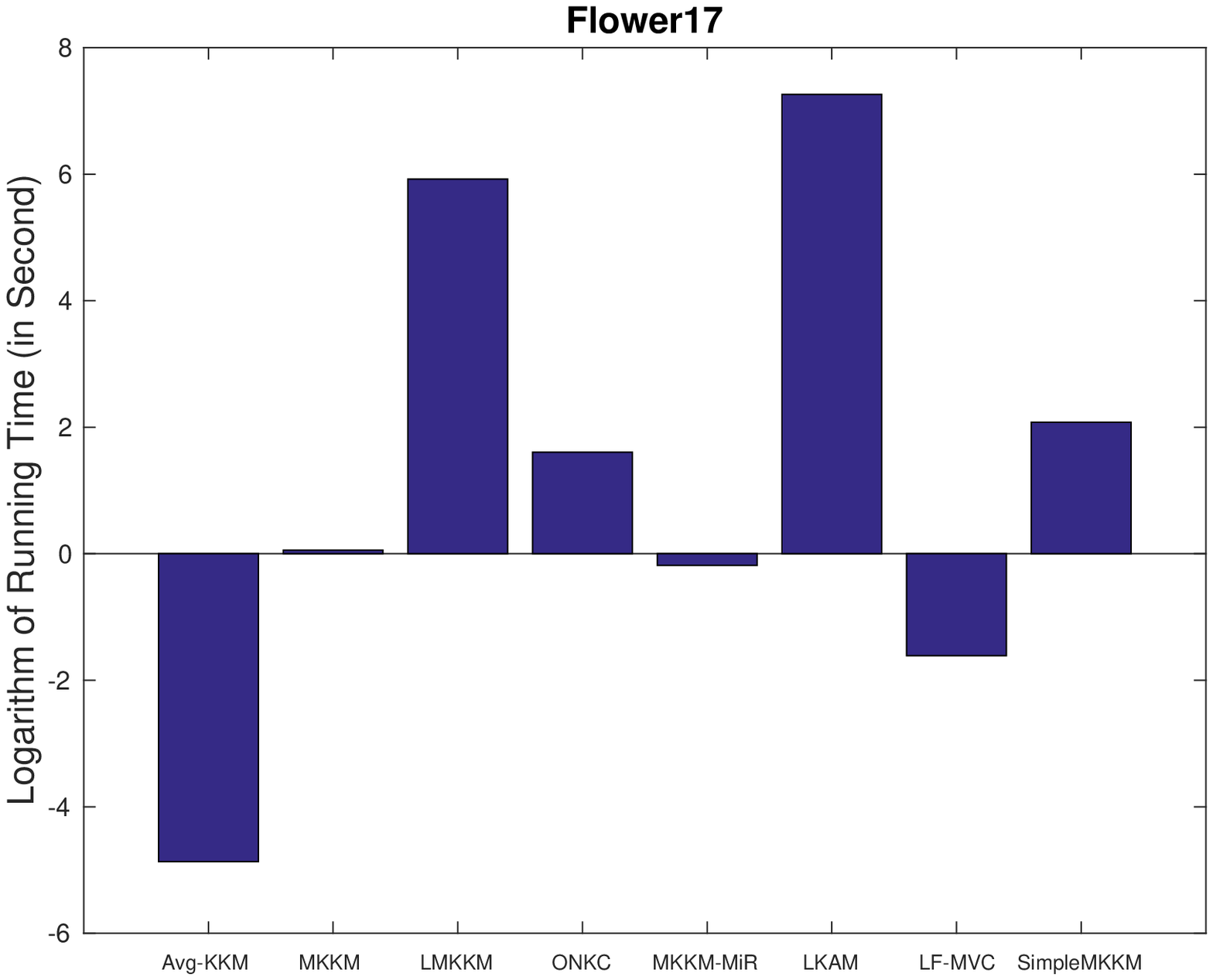}\label{RunningTime_Flower17}}%
\subfigure{\includegraphics[width=0.33\textwidth]{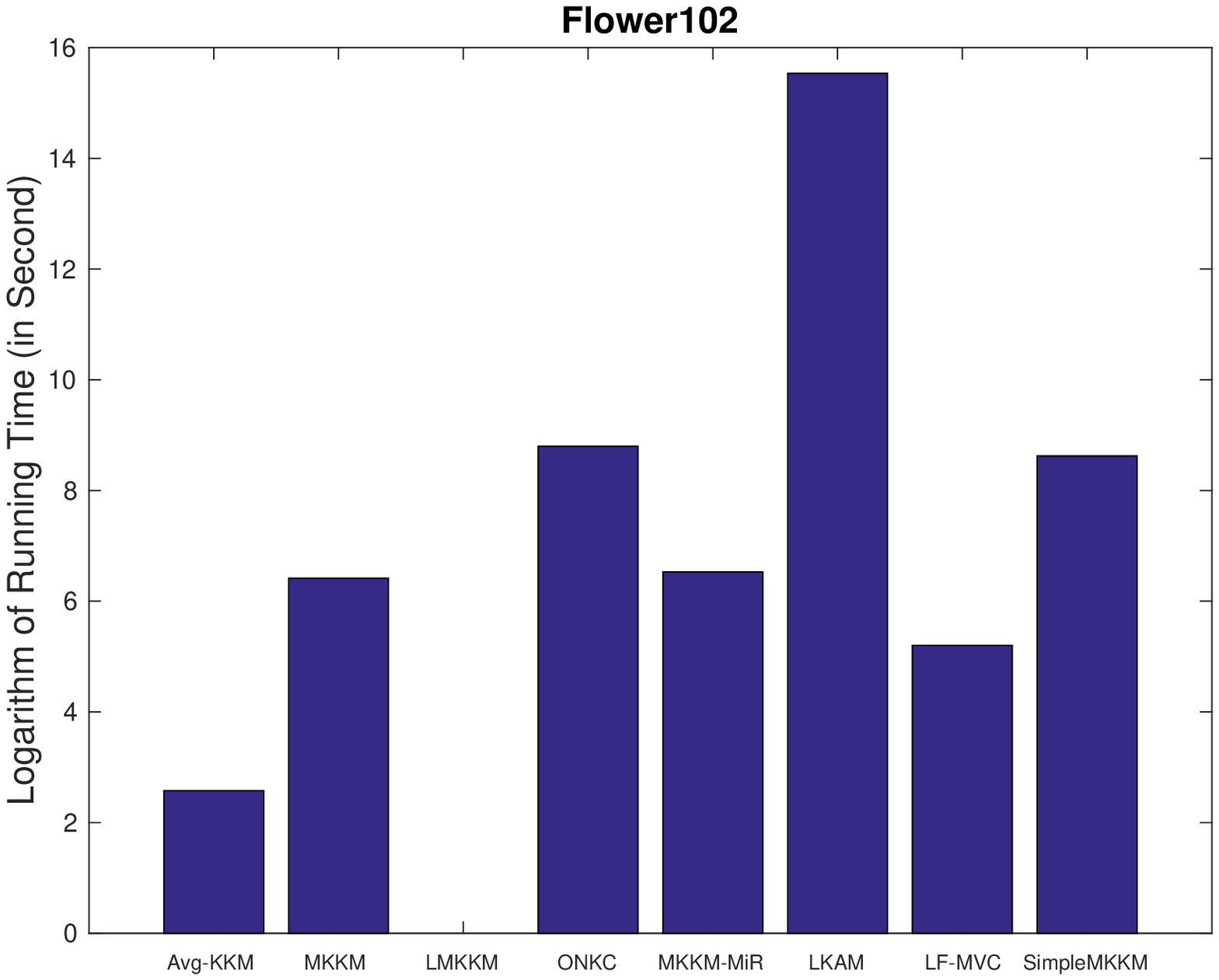}\label{RunningTime_Flower102}}%
\subfigure{\includegraphics[width=0.33\textwidth]{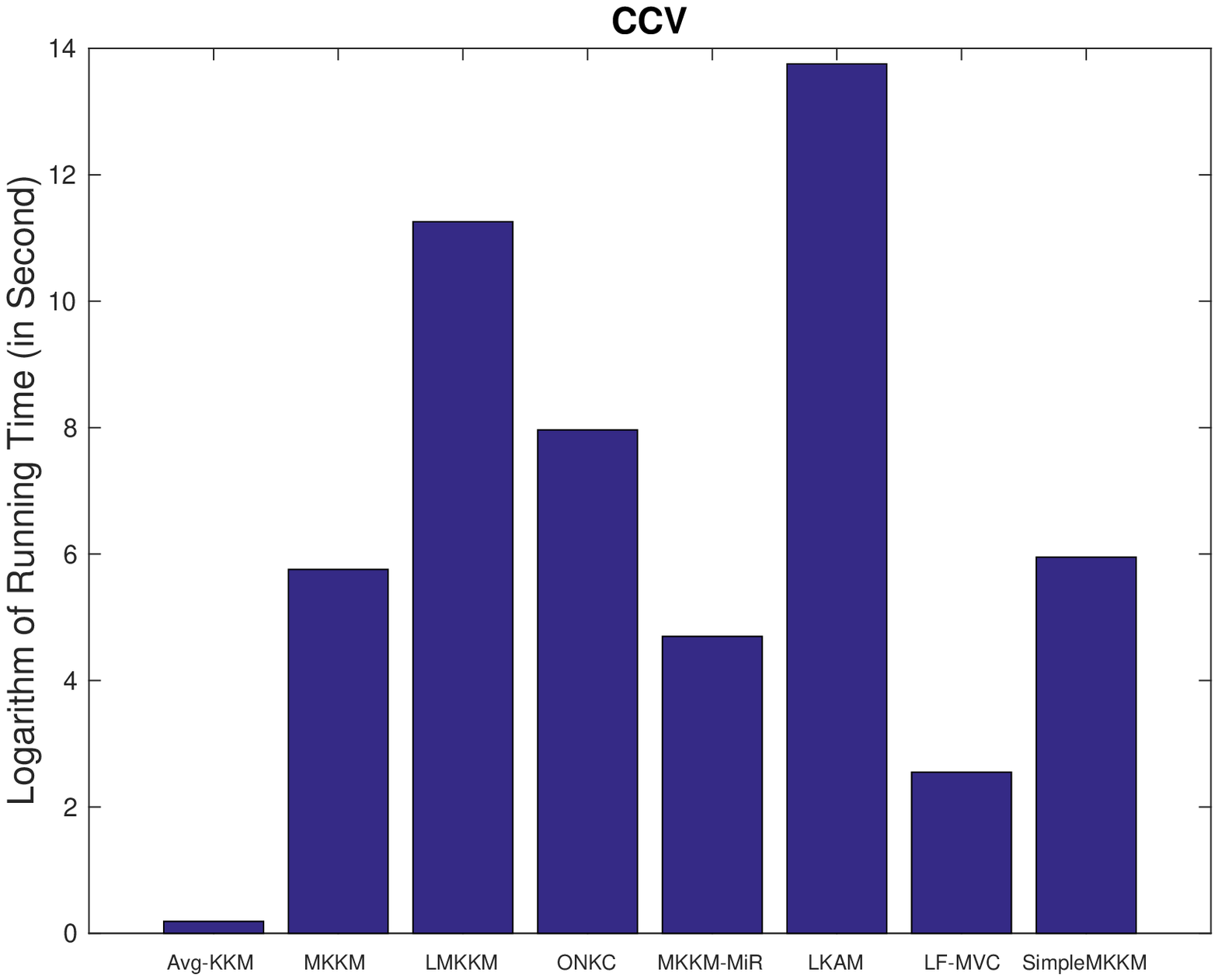}\label{RunningTime_CCV}}\\[-20pt]
%\subfigure{\includegraphics[width=0.25\textwidth]{myFigures/KernelWeights_Caltech102-5.eps}\label{flower17_accIter}}%
%\subfigure{\includegraphics[width=0.25\textwidth]{myFigures/KernelWeights_Caltech102-10.eps}\label{Flower102_accIter}}%
%\subfigure{\includegraphics[width=0.25\textwidth]{myFigures/KernelWeights_Caltech102-15.eps}\label{UCIdigtal_accIter}}%
%\subfigure{\includegraphics[width=0.25\textwidth]{myFigures/KernelWeights_Caltech102-20.eps}\label{CCV_accIter}}\\
%\subfigure{\includegraphics[width=0.25\textwidth]{myFigures/KernelWeights_Caltech102-25.eps}\label{Caltech102_accIter}}%
\subfigure{\includegraphics[width=0.33\textwidth]{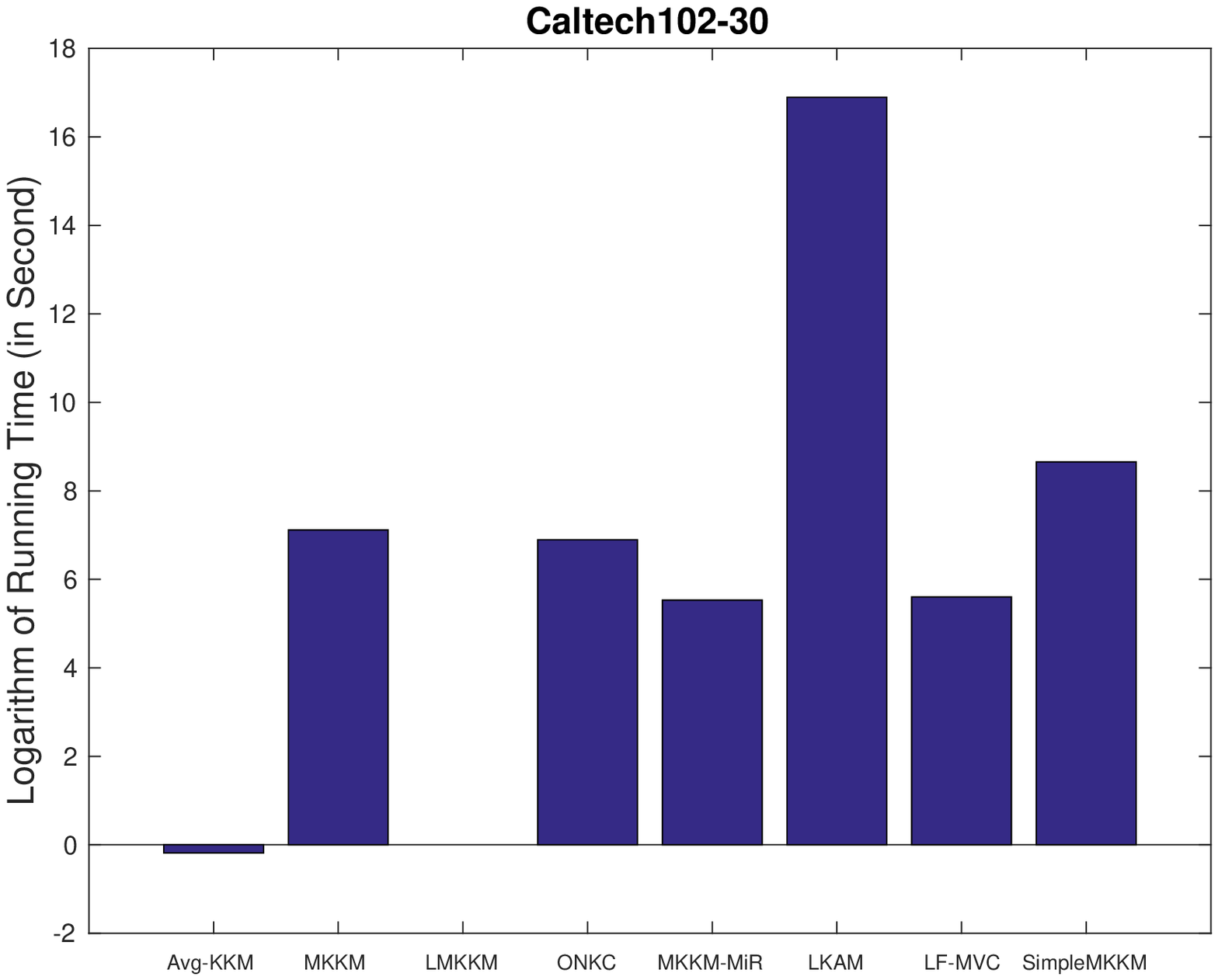}\label{RunningTime_Caltech102-30}}%
\subfigure{\includegraphics[width=0.33\textwidth]{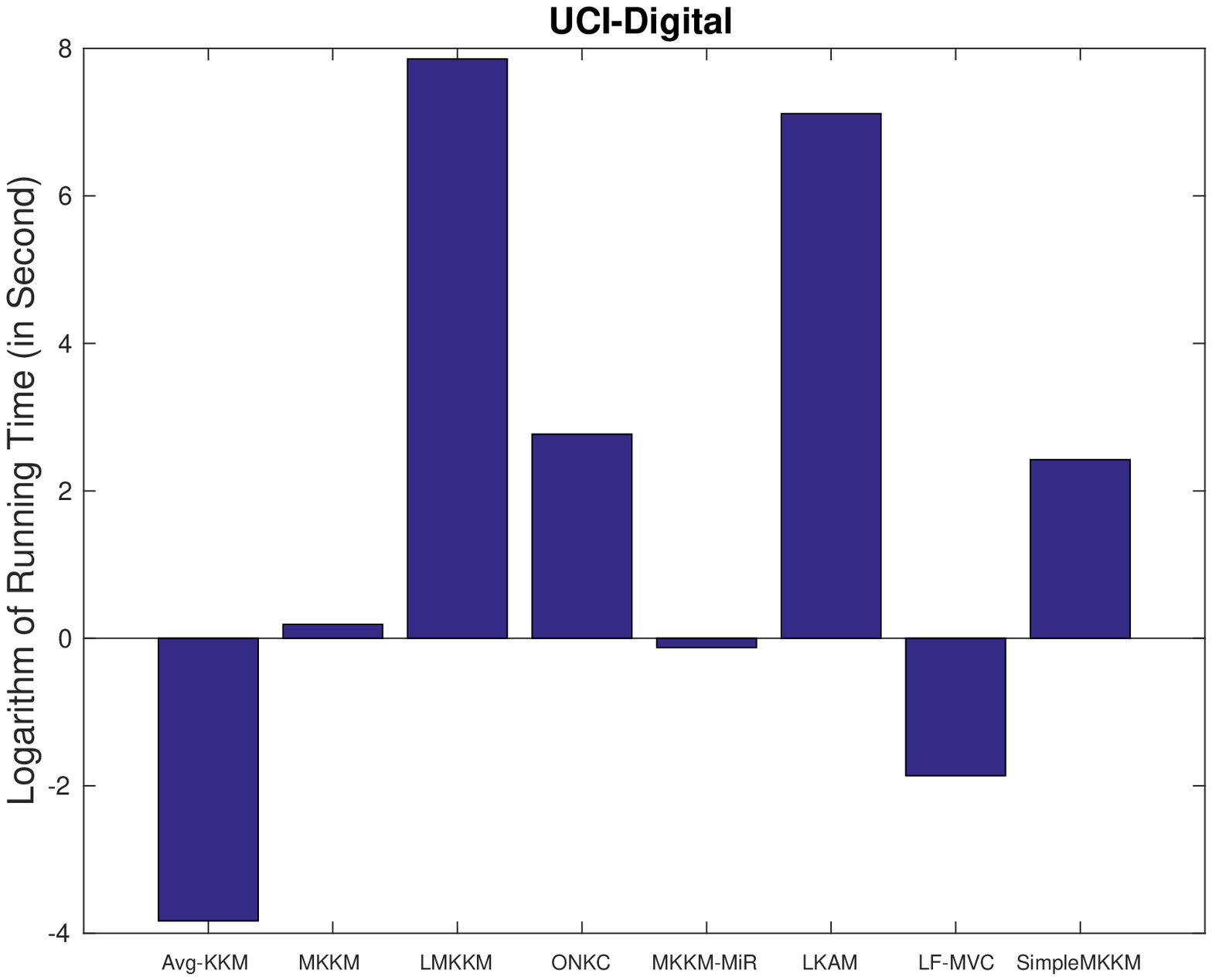}\label{RunningTime_UCIdigital}}%
\subfigure{\includegraphics[width=0.33\textwidth]{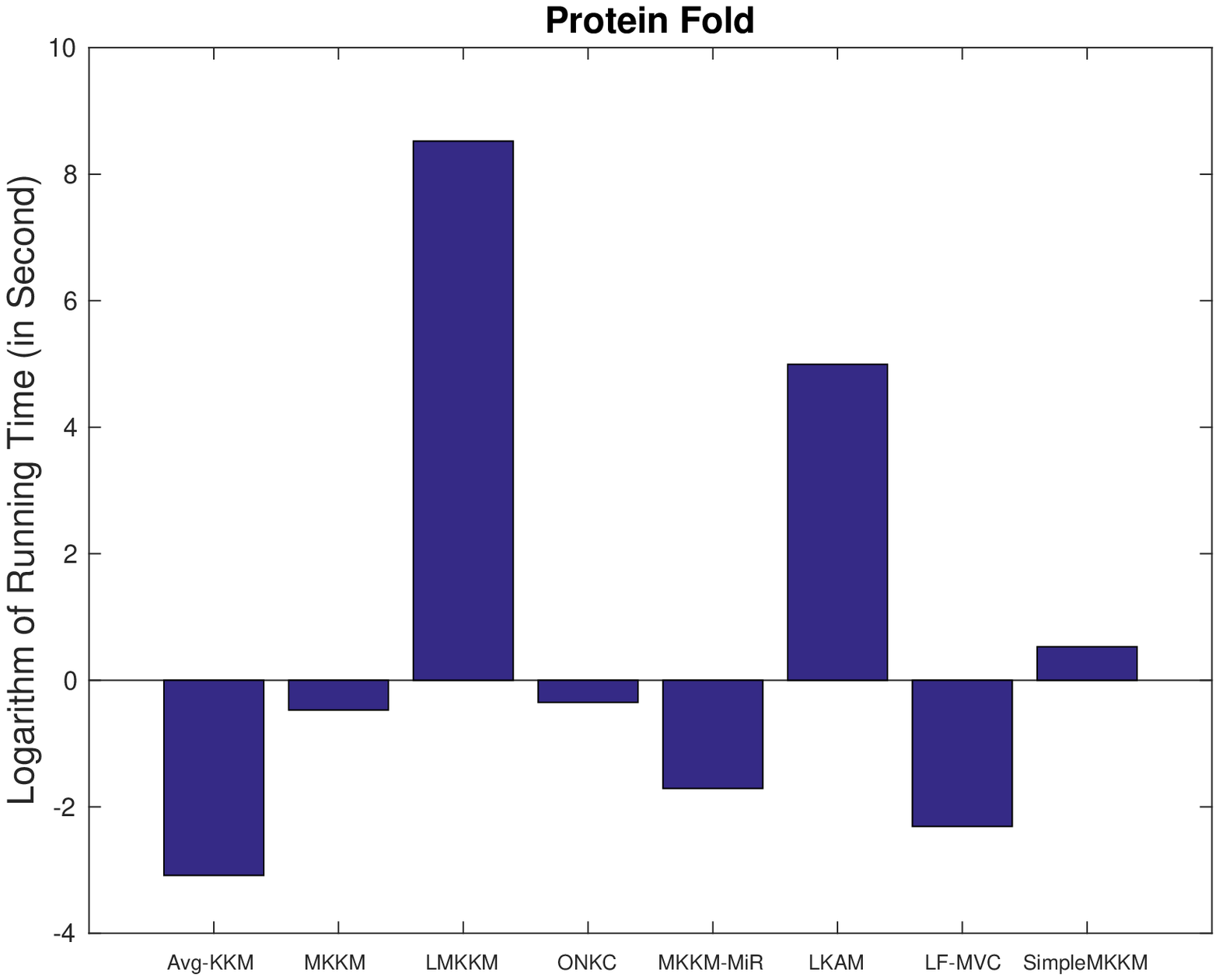}\label{RunningTime_ProteinFold}}\\
\vspace{-18pt}
\caption{Run time of different algorithms on six benchmark datasets (in seconds). The experiments are conducted on a PC  with  Intel(R) Core(TM)-i7-5820 3.3 GHz CPU and 32G RAM in MATLAB environment. SimpleMKKM is comparably fast to alternatives while providing superior performance and requiring no hyper-parameter tuning. Results for other datasets are omitted due to space limit.}\label{FigRunningTime}
\vspace{-8pt}
\end{figure*}

\begin{figure*}[!ht]
\vspace{-5pt}
\centering
\subfigure{\includegraphics[width=0.165\textwidth]{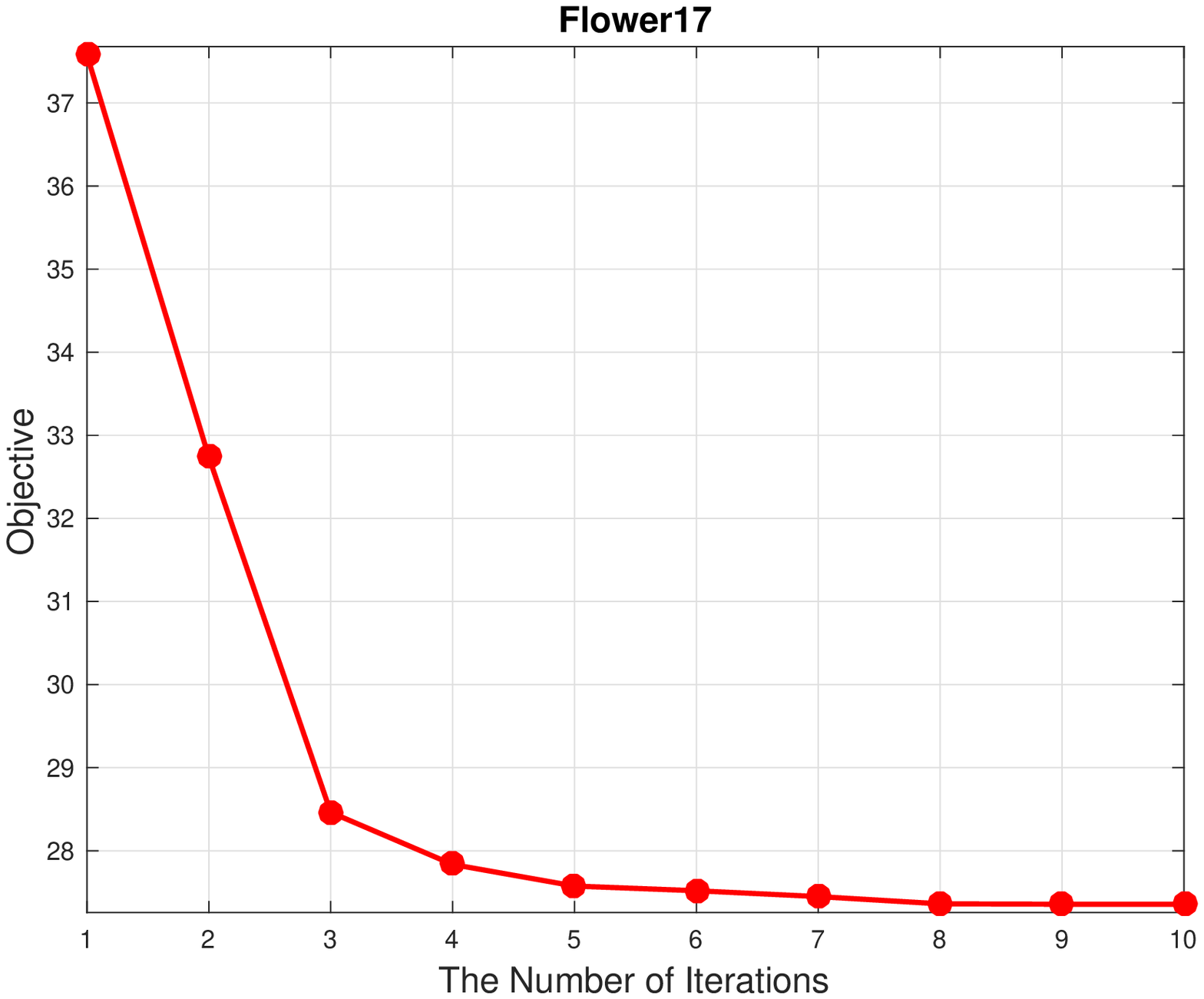}\label{Objective_Flower17}}%
\subfigure{\includegraphics[width=0.165\textwidth]{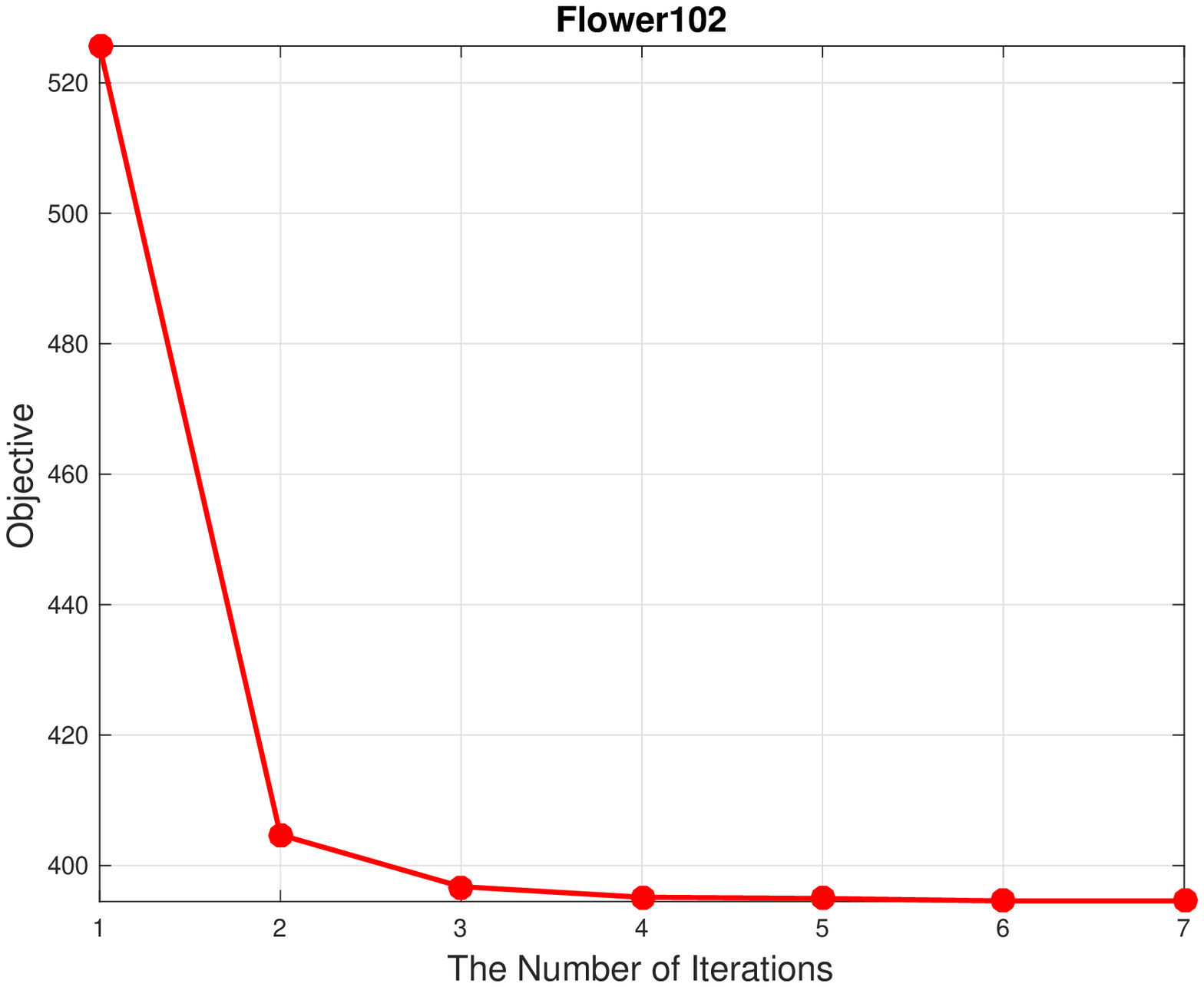}\label{Objective_Flower102}}%
\subfigure{\includegraphics[width=0.165\textwidth]{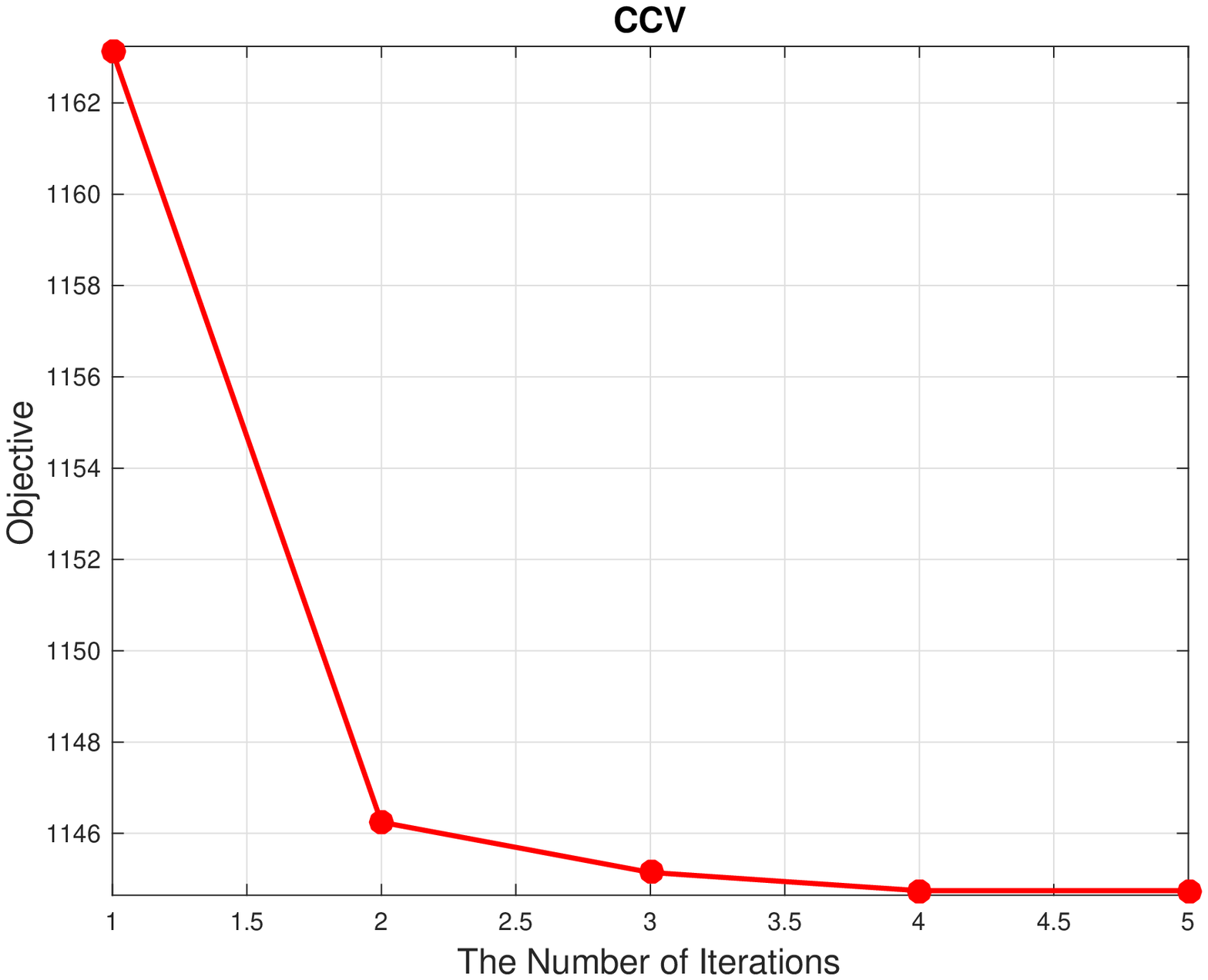}\label{Objective_CCV}}%
%\subfigure{\includegraphics[width=0.25\textwidth]{myFigures/KernelWeights_Caltech102-5.eps}\label{flower17_accIter}}%
%\subfigure{\includegraphics[width=0.25\textwidth]{myFigures/KernelWeights_Caltech102-10.eps}\label{Flower102_accIter}}%
%\subfigure{\includegraphics[width=0.25\textwidth]{myFigures/KernelWeights_Caltech102-15.eps}\label{UCIdigtal_accIter}}%
%\subfigure{\includegraphics[width=0.25\textwidth]{myFigures/KernelWeights_Caltech102-20.eps}\label{CCV_accIter}}\\
%\subfigure{\includegraphics[width=0.25\textwidth]{myFigures/KernelWeights_Caltech102-25.eps}\label{Caltech102_accIter}}%
\subfigure{\includegraphics[width=0.165\textwidth]{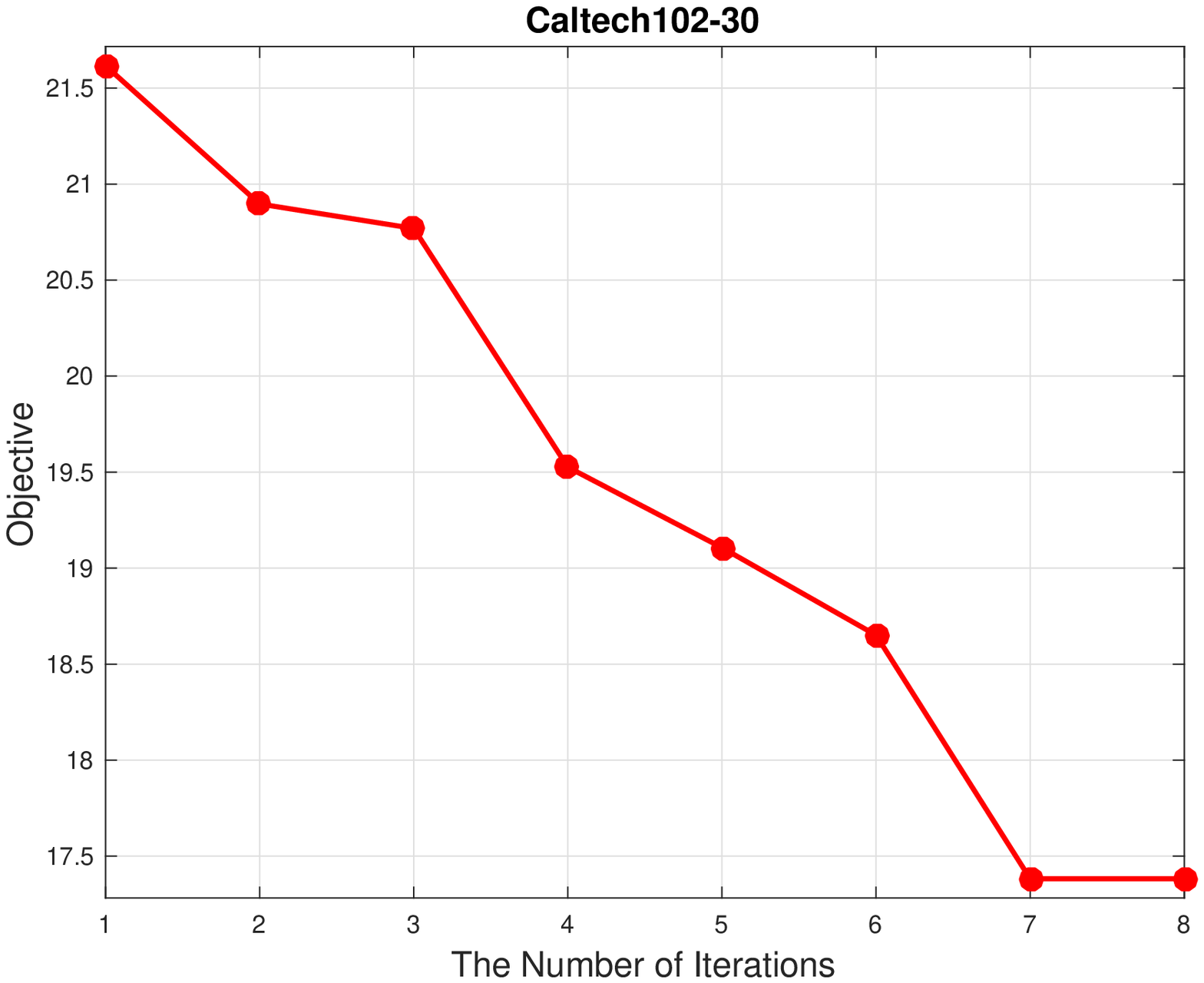}\label{Objective_Caltech102-30}}%
\subfigure{\includegraphics[width=0.165\textwidth]{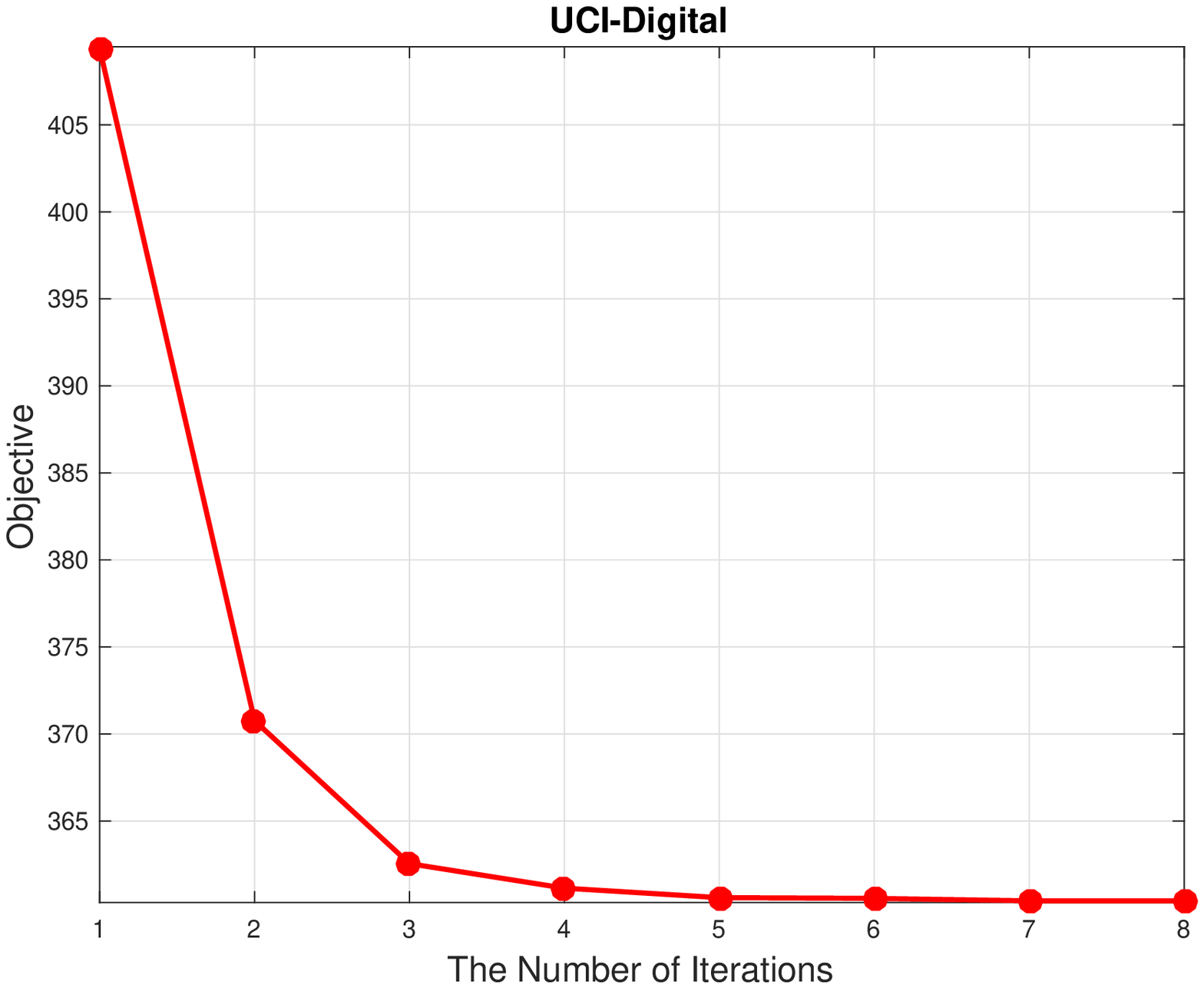}\label{Objective_UCIdigital}}%
\subfigure{\includegraphics[width=0.165\textwidth]{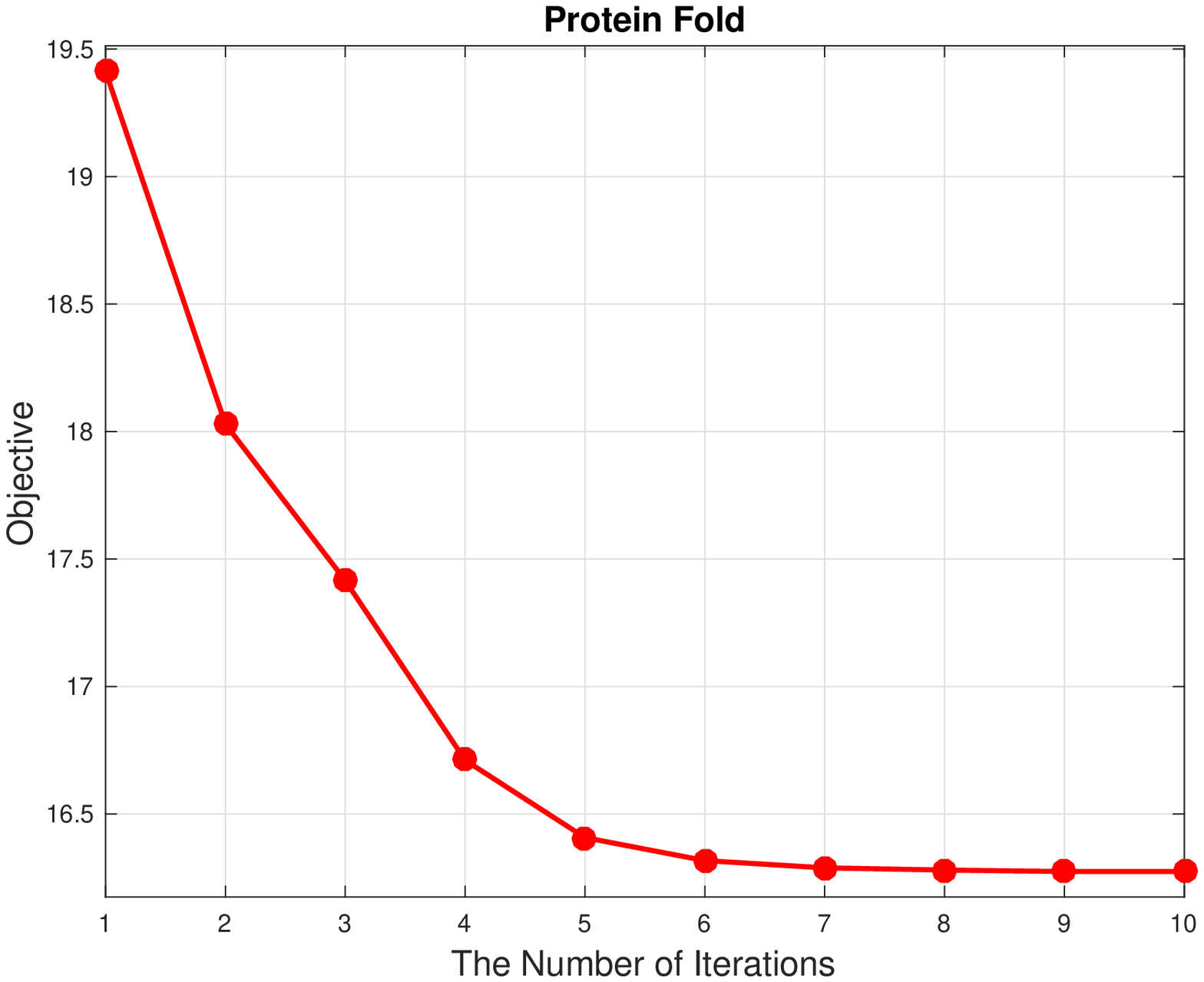}\label{Objective_ProteinFold}}\\
\vspace{-10pt}
\caption{The objective of SimpleMKKM decreases with iterations. The curves for other datasets are omitted due to space limit.}\label{FigObjective}
\vspace{-10pt}
\end{figure*}

For all algorithms, we repeat each experiment $50$ times with random initialization to reduce the effect of randomness caused by \kmeans, and report the means and variation. We next thoroughly study  SimpleMKKM in terms of: clustering performance, the learned kernel weights, running time and algorithm convergence.

Along with SimpleMKKM, we ran another eight comparative algorithms in recent MKC literature, including
\vspace{-10pt}
\begin{itemize}
\item \textbf{Average kernel \kmeans{} (Avg-KKM)}. The consensus kernel is the uniformly combined base kernels, which is taken as the input of kernel \kmeans.
\vspace{-5pt}
\item \textbf{Multiple kernel \kmeans{} (MKKM)} \cite{HuangCC12}. The base kernels are linearly combined into the consensus kernel. In addition, the combination weights are optimized along with clustering.
\vspace{-5pt}
\item \textbf{Localized multiple kernel \kmeans (LMKKM)} \cite{GonenM14}. The base kernels are combined with sample-adaptive weights.
\vspace{-5pt}
\item \textbf{Optimal neighborhood kernel clustering (ONKC)} \cite{LiuZWLDZY17}. The consensus kernel is chosen from the neighbor of linearly combined base kernels.
\vspace{-5pt}
\item \textbf{Multiple kernel \kmeans with matrix-induced regularization (MKKM-MiR)} \cite{LiuDY0Z16}. The optimal combination weights are learned by introducing a matrix-induced regularization term to reduce the redundancy among the base kernels.
\vspace{-5pt}
\item \textbf{Mulitple kernel clustering with local alignment maximization (LKAM)} \cite{LiL0DYZ16}. The similarity of a sample to its $k$-nearest neighbors, instead of all samples, is aligned with the ideal similarity matrix.
\vspace{-5pt}
\item \textbf{Multi-view clustering via late fusion alignment maximization (LF-MVC)} \cite{WangLZTLHXY19}. Base partitions are first computed within corresponding data views and then integrated into a consensus partition.
\vspace{-5pt}
\item \textbf{MKKM-MM} \cite{Bang2018robust}. It proposes a $\min_{\mathbf{H}}$-$\max_{\boldsymbol{\gamma}}$ formulation that combines views in a way to reveal high within-cluster variance in the combined kernel space and then updates clusters by minimizing such variance.
\end{itemize}
The implementations of the above algorithms are publicly available in corresponding papers, and we directly adopt them without revision in our experiments. Among all the compared algorithms, ONKC \cite{LiuZWLDZY17}, MKKM-MiR \cite{LiuDY0Z16}, LKAM \cite{LiL0DYZ16} and LF-MVC \cite{WangLZTLHXY19} have hyper-parameters to be tuned.  We reuse their released Matlab codes and carefully tuned the hyper-parameters according to their setup to produce the best possible results on each dataset.

\begin{table}[!ht]
\centering
\vspace{-2pt}
\caption{Empirical comparison of SimpleMKKM with KAMM-R and KAMM-A on Flower17.}\label{ClusteringACC2}
\vspace{-2pt}
\begin{center}
\begin{small}
\begin{sc}
\resizebox{0.8\linewidth}{!}{
\begin{tabular}{lccc}
\toprule
\scriptsize{Dataset}     &  \scriptsize{KAMM-R} & \scriptsize{KAMM-A} & \scriptsize{SimpleMKKM}\\
\toprule
ACC  & {35.0$\pm$ 0.4} & {54.1$\pm$ 1.8} & \textbf{58.9$\pm$ 1.3}  \\
NMI & {37.6$\pm$ 0.3} & {54.1$\pm$ 1.4} & \textbf{57.3$\pm$ 0.8}  \\
Purity & {36.6 $\pm$ 0.5} & {55.1$\pm$ 1.8} &  \textbf{60.2$\pm$ 1.4} \\
\bottomrule
\end{tabular}}
\end{sc}
\end{small}
\end{center}
\vspace{-8pt}
\end{table}

\subsection{Experimental Results}

\paragraph{Clustering Performance}
Table \ref{ClusteringACC} presents the ACC, NMI and purity comparison of the above algorithms. From this table, we have the following observations:
\vspace{-8pt}
\begin{itemize}
\item The proposed SimpleMKKM consistently and significantly outperforms MKKM. For example, it exceeds MKKM by $12.7\%,\,16\%,\,6.1\%,\,3.1\%,\,34.6\%,\,4.4\%,\,7.2\%,\,8.9\%,$\\ $10.1\%,\,10.6\%$ and $11.7\%$ in terms of ACC on all benchmark datasets. These results demonstrate the efficacy of its min-max formulation and associated optimization algorithm.
\vspace{-4pt}
\item MKKM-MM \cite{Bang2018robust} is the first try in literature to improve MKKM via minimization-maximization. As observed, it does improve the MKKM. However the improvement over MKKM is marginal on all datasets. Meanwhile, the proposed SimpleMKKM significantly outperforms MKKM-MM. This once again demonstrates the advantage of our formulation and the associated optimization strategy.
\vspace{-4pt}
\item Our SimpleMKKM achieves comparable or slightly better performance than MKKM-MiR \cite{LiuDY0Z16},  ONKC \cite{LiuZWLDZY17}, and LF-MVC \cite{WangLZTLHXY19}, all of which are considered the state of the art in multi-kernel clustering. Note that all of these algorithms have several hyper-parameters to tune due to the incorporation of regularization on the kernel weight $\boldsymbol{\gamma}$. Though demonstrating promising clustering performance, these algorithms need to take a lot of effort to determine the best hyper-parameters in practical applications. And parameter tuning may be impossible in real applications where there is no ground truth clustering to optimize. In contrast, our SimpleMKKM is parameter-free.
\end{itemize}

In summary, SimpleMKKM demonstrates superior clustering performance over the alternatives on all datasets and has no hyper-parameter to be tuned. We expect that the simplicity and efficacy of SimpleMKKM will make it a good option to be considered for practical clustering applications. Note that some results of LMKKM \cite{GonenM14} are not reported due to out-of-memory errors, which are caused by its cubic computational and memory complexity.

\paragraph{Advantage of Formulation and Optimization} In order to show the advantage of the proposed formulation and optimization algorithm, we conduct an extra experiment on Flower17 to compare alternatives KAMM-R and KAMM-A. KAMM-R denotes optimizing kernel alignment $\mathrm{\mathbf{K}_{\boldsymbol{\gamma}\mathbf{H}\mathbf{H}^{\top}}}$ via maximizing $\boldsymbol{\gamma}$ and maximizing $\mathbf{H}$ with reduced gradient descent, and KAMM-A denotes optimizing this criterion via minimizing $\boldsymbol{\gamma}$ and maximizing $\mathbf{H}$ with alternate optimization (see Section~\ref{sec:formulation} for discussion). %The objective of KAMM-R is $\max_{\boldsymbol{\gamma}}\max_{\mathbf{H}}\mathrm{\mathbf{K}_{\boldsymbol{\gamma}\mathbf{H}\mathbf{H}^{\top}}}$, which can be solved by the aforementioned reduced gradient descent algorithm.
KAMM-A has the same objective as SimpleMKKM, but it uses the widely adopted alternate optimization to solve it in place of our newly derived reduced gradient algorithm. From the results reported in Table~\ref{ClusteringACC2}, we clearly observe that: (1) Our SimpleMKKM formulation has significant advantage over KAMM-R, demonstrating the value of our novel min-max objective; (2) It also outperforms KAMM-A, which confirms that our new gradient-based optimization algorithm is also much better than the widely used alternate optimization. %Note that by alternate optimization, the objective of KAMM-A cannot be guaranteed to monotonically decrease with iteration. In such case, we stop the algorithm when the number of iteration exceeds 50 in our experiment.

\paragraph{Kernel Weight Analysis}
We next investigate the kernel weights learned by the compared algorithms. The results are plotted in Figure \ref{FigKernelWeights}. We can see that the kernel weights learned by MKKM are extremely sparse on some datasets such as UCI-Digital, which is caused by the alternate optimization. This sparsity insufficiently exploits the multiple kernel matrices  and explains the weak performance of MKKM. For example, the clustering accuracy of MKKM on UCI-Digital is only $47.2\%$. However, despite the $\ell_{1}$-norm constraint on $\boldsymbol{\gamma}$, the kernel weights learned by our SimpleMKKM are all non-sparse on all datasets, which contributes to its superior clustering performance. This non-sparsity of the learned kernel weights is attributed to our new reduced gradient descent algorithm, which in turn is derived based on our new min-max kernel alignment objective.

\paragraph{Runtime and Convergence}
We also report the running time of the compared algorithms in Figure \ref{FigRunningTime}. As observed, in addition to significantly improving performance, SimpleMKKM does not considerably increase the running time compared with MKKM and its variants. The objective of SimpleMKKM with iterations is reported in Figure~\ref{FigObjective}. From these figures, we observe that the objective is monotonically decreased and the algorithm usually converges in less than ten iterations on all datasets. This corroborates our earlier theoretical analysis of the nature of our proposed objective and efficient optimisation algorithm.

\section{Conclusion}

In this paper, we have extended the widely used supervised kernel alignment criterion to clustering, and introduce a novel clustering objective of by minimizing alignment for $\boldsymbol{\gamma}$ and maximizing it for $\mathbf{H}$. We show that this novel objective can be transformed into a minimization problem which is differentiable and amenable to a solution by reduced gradient descent. This makes SimpleMKKM unique among MKC alternatives, in not requiring a local-minimum prone alternating coordinate descent strategy.

We derive a generalization bound for our approach using global Rademacher complexity analysis. Comprehensive experiments  demonstrate the effectiveness of SimpleMKKM. We expect that the simplicity, lack of hyper-parameters, and efficacy of SimpleMKKM will make it a go-to solution for practical multi-kernel clustering applications in future.
%It is expected that the simplicity and effectiveness of SimpleMKKM will make it a good option to be considered for practical multiple kernel clustering applications.
Future work may aim to extend SimpleMKKM to handle incomplete kernels, study further applications, and derive convergence rates using local Rademacher complexity analysis \cite{KloftB12,CortesKM13}.
\bibliography{myICML2020}
\bibliographystyle{icml2020}

\end{document}